\documentclass{article}

    \PassOptionsToPackage{numbers, square}{natbib}

\usepackage[final]{neurips_2020}

\usepackage[utf8]{inputenc} 
\usepackage[T1]{fontenc}    
\usepackage{hyperref}       
\usepackage{url}            
\usepackage{booktabs}       
\usepackage{amsfonts}       
\usepackage{nicefrac}       
\usepackage{microtype}      

\usepackage{amssymb}
\usepackage{amsmath}
\usepackage{amsthm}

\usepackage{bm}
\usepackage{color}
\usepackage{graphicx}
\usepackage{tikz}

\newcommand{\customfootnotetext}[2]{{
  \renewcommand{\thefootnote}{#1}
  \footnotetext[0]{#2}}}

\usepackage[compact]{titlesec}
    \titlespacing{\section}{0pt}{1.5ex}{.75ex}
    \titlespacing{\subsection}{0pt}{1ex}{0ex}
    \titlespacing{\subsubsection}{0pt}{0.5ex}{0ex}

\usepackage{subcaption}

\newlength{\commentWidth}
\usetikzlibrary{shapes,backgrounds,calc}

\makeatletter
\tikzset{circle split part fill/.style  args={#1,#2}{%
 alias=tmp@name, 
  postaction={%
    insert path={
     \pgfextra{%
     \pgfpointdiff{\pgfpointanchor{\pgf@node@name}{center}}%
                  {\pgfpointanchor{\pgf@node@name}{east}}%
     \pgfmathsetmacro\insiderad{\pgf@x}
      \fill[#1] (\pgf@node@name.base) ([xshift=-\pgflinewidth]\pgf@node@name.east) arc
                          (0:180:\insiderad-\pgflinewidth)--cycle;
      \fill[#2] (\pgf@node@name.base) ([xshift=\pgflinewidth]\pgf@node@name.west)  arc
                           (180:360:\insiderad-\pgflinewidth)--cycle;            
         }}}}}  
 \makeatother  

\usepackage[ruled,noend]{algorithm2e} 

\SetAlFnt{\small}
\SetAlCapFnt{\small}
\SetAlCapNameFnt{\small}
\SetAlCapHSkip{0pt}

\newcommand{\matching}[0]{matching}
\newcommand{\Matching}[0]{Matching}

\newcommand{\indep}{\perp \!\!\! \perp}

\newcommand{\MCTSS}{\texttt{MCTS}}
\newcommand{\GREEDYS}{\texttt{Greedy}}
\newcommand{\OPT}{\texttt{OPT}}

\usepackage{paralist}

\newtheorem{theorem}{Theorem}[section]
\newtheorem{lemma}[theorem]{Lemma}
\newtheorem{assumption}[theorem]{Assumption}
\newtheorem{proposition}[theorem]{Proposition}

\newtheorem{definition}[theorem]{Definition}

\DeclareMathOperator*{\argmax}{arg\,max}

\DeclareMathOperator*{\expect}{\mathbb E}
\DeclareMathOperator*{\prob}{\mathbb P}

\title{Improving Policy-Constrained Kidney Exchange \\ via Pre-Screening}

\author{%
  Duncan C McElfresh\\
  Mathematics Department\\
    Computer Science Department\\
  University of Maryland\\
  College Park, MD 20742\\
  \texttt{dmcelfre@umd.edu} \\
   \And
    Michael Curry\\
    Computer Science Department\\
   University of Maryland\\
   College Park, MD 20742\\
    \texttt{curry@cs.umd.edu} \\
   \AND
    Tuomas Sandholm \\
    Computer Science Department\\
    Carnegie Mellon University\\
    Strategy Robot, Inc. \\
   Optimized Markets, Inc. \\
  Strategic Machine, Inc. \\
   \And
   John P Dickerson \\
    Computer Science Department\\
  University of Maryland\\
  College Park, MD 20742\\
  \texttt{john@cs.umd.edu} \\
}

\begin{document}

\maketitle

\begin{abstract}
In barter exchanges, participants swap goods with one another without exchanging money; these exchanges are often facilitated by a central clearinghouse, with the goal of maximizing the aggregate quality (or number) of swaps. 
Barter exchanges are subject to many forms of uncertainty--in participant preferences, the feasibility and quality of various swaps, and so on. 
Our work is motivated by kidney exchange, a real-world barter market in which patients in need of a kidney transplant swap their willing living donors, in order to find a better match.
Modern exchanges include 2- and 3-way swaps, making the kidney exchange \emph{clearing problem} NP-hard.
Planned transplants often \emph{fail} for a variety of reasons--if the donor organ is rejected by the recipient's medical team, or if the donor and recipient are found to be medically incompatible. 
Due to 2- and 3-way swaps, failed transplants can ``cascade'' through an exchange; one US-based exchange estimated that about $85\%$ of planned transplants failed in 2019.
Many optimization-based approaches have been designed to avoid these failures; however most exchanges cannot implement these methods, due to legal and policy constraints.
Instead, we consider a setting where exchanges can \emph{query} the preferences of certain donors and recipients--asking whether they would accept a particular transplant.
We characterize this as a two-stage decision problem, in which the exchange program (a) queries a small number of transplants before committing to a matching, and (b) constructs a matching according to fixed policy. 
We show that selecting these edges is a challenging combinatorial problem, which is non-monotonic and non-submodular, in addition to being NP-hard. 
We propose both a greedy heuristic and a Monte Carlo tree search, which outperforms previous approaches, using experiments on both synthetic data and real kidney exchange data from the United Network for Organ Sharing.
\end{abstract}

\section{Introduction}\label{sec:intro}

We consider a multi-stage decision problem in which a decision-maker uses a fixed \emph{policy} to solve a hard (stochastic) problem. 
Before using the policy, 
%
%
the decision-maker can first \emph{measure} some of the uncertain problem parameters--in a sense, guiding the policy toward a better solution. 
%
%
Our primary motivation is kidney exchange, a process where patients in need of a kidney transplant swap their (willing) living donors, in order to find a better match. 
Many government-run kidney exchanges match patients and donors using a \emph{\matching{} algorithm} that follows strict policy guidelines~\cite{biro2019modelling}; this \matching{} algorithm is often written into law or policy, and is not easily modified. 
%
%
Modern kidney exchanges use both cyclical swaps and chain-like structures (initiated by an unpaired altruistic donor)~\cite{Rees09:Nonsimultaneous}, and identifying the max-size or max-weight set of transplants is both NP- and APX-hard~\cite{Abraham07:Clearing,Biro09:Maximum}.

In kidney exchange--as in many resource allocation settings--information used by the decision-maker is subject to various forms of uncertainty. 
Here we are primarily concerned with uncertainty in the \emph{feasibility} of potential transplants: if a donor is matched with a potential recipient, will the transplant actually occur?
Planned transplants may \emph{fail} for a variety of reasons: for example, medical testing may reveal that the donor and recipient are incompatible (a \emph{positive crossmatch}); the recipient or their medical team may reject a donor organ in order to wait for a better match; or the donor may decide to donate elsewhere before the exchange is planned. 
Failed transplants are especially troublesome in kidney exchange, due to the cycle and chain structures used: for example, suppose that a cyclical swap is planned between three patient/donor pairs; if any one of the planned transplants fails, then none of the other transplants in that cycle can occur.
Unfortunately, it is quite common for planned transplants to fail.  
For example, the United Network for Organ Sharing (UNOS\footnote{UNOS is the organization tasked with overseeing organ transplantation in the US: \url{https://unos.org/}.}) estimates that in FY2019, about $85\%$ of their planned kidney transplants failed~\cite{Leishman19:Challenges}.

Various \matching{} algorithms have been proposed that aim to mitigate transplant failures (for example, using stochastic optimization~\cite{dickerson2019failure,Anderson15:Finding}, robust optimization~\cite{mcelfresh2019scalable}, or conditional value at risk~\cite{bidkhori2020kidney}). 
However, implementing these strategies would require modifying fielded \matching{} algorithms--which in many cases would require changing law or policy.
One way to avoid failures without modifying the \matching{} algorithm is to \emph{pre-screen} potential transplants~\cite{Leishman19:Challenges,Blum13:Harnessing,blum2020ignorance}, by communicating with the recipients' medical team and possibly using additional medical tests.
Pre-screening transplants is costly, as it requires scarce time and resources.
Furthermore, there are often many thousand potential transplants in any given exchange; selecting which transplants to screen is not easy. 

In this paper we investigate methods for selecting a limited number of transplants to pre-screen, in order to ``guide'' the \matching{} algorithm to a better outcome.
We formalize this as a multistage stochastic optimization problem, and we consider both an \emph{offline} setting (where screenings are selected all at once), and an \emph{online} setting (where screenings are selected sequentially).
%
%
%
%

\noindent \textbf{Related Work.}
%
%
While kidney exchange is known to be a hard packing problem, several algorithms exist that are scalable in practice, and are used by fielded exchanges~\cite{Dickerson16:Position,Anderson15:Finding,Manlove15:Paired}.
Prior work has addressed potential transplant failures; our model is inspired by~\citet{dickerson2019failure}.
Pre-screening potential transplants has also been addressed in prior work~(\cite{blum2020ignorance,Molinaro13:Kidney}, and \S~5.1 of~\cite{dickerson2016unified}), and our model is similar to stochastic matching and stochastic $k$-set packing~\cite{bansal2012lp}.
However there are substantial differences between these models and ours:
(a) many prior approaches assume that a large number of transplants may be pre-screened~\cite{blum2020ignorance,Molinaro13:Kidney}--on the order of one for each patient in the exchange; we assume far fewer screenings are possible;
%
%
(b) prior work often assumes a \emph{query-commit} setting--where successfully pre-screened transplants \emph{must} be matched.
Instead we assume that non-screened transplants may also be matched--which more-accurately represents the way that modern exchanges operate;
(c) most prior work assumes that transplants that pass pre-screening are guaranteed to result in a transplant.
In reality, transplants often fail after pre-screening, a fact reflected in our model.


One of our approaches is based on \emph{Monte Carlo Tree Search} (MCTS), which allows efficient exploration of intractably large decision trees.
While MCTS is primarily associated with Markov decision processes and game-playing~\cite{bouzy2004associating}, it has been used successfully for combinatorial optimization~\cite{kartal2016monte}.
We use a version of MCTS, Upper Confidence Bounds for Trees (UCT), which balances exploration and exploitation by treating each tree node as a multi-armed bandit problem~\cite{auer2002finite,kocsis2006bandit}.

\subsection*{Our Contributions}

\begin{enumerate}
    \item (\S~\ref{sec:edge-selection-problem}) We formalize the \emph{policy-constrained edge query problem}: where a decision-maker (such as a kidney exchange program) selects a set of potential edges (potential transplants) to pre-screen, prior to constructing a final packing (a set of transplants) using a fixed algorithm.
    This model generalizes existing models in the literature, as edge failure probabilities depend on whether or not the edge is pre-screened.
    Further, we allows for context-specific constraints, such as those imposed by public policy or the particular hospital or exchange.
    \item (\S~\ref{sec:solving-the-query-problem}) We prove that when the decision-maker uses a max-weight packing policy (the most common choice among fielded exchanges), the edge query problem is both non-monotonic and non-submodular in the set of queried edges.
    Despite these worst-case findings we show that this problem is nearly monotonic for real and synthetic data, and simple algorithms perform quite well.
    On the other hand, when the decision-maker uses a \emph{failure-aware} (stochastic) packing policy, the edge query problem becomes monotonic under mild assumptions.
    \item (\S~\ref{sec:expt}) We conduct numerical experiments on both simulated and real exchange data from the United Network for Organ Sharing (UNOS). We demonstrate that our methods substantially outperform prior approaches and a randomized baseline.
    %
\end{enumerate}


%
%

\section{The Policy-Constrained Edge Query Problem}\label{sec:edge-selection-problem}

Kidney exchanges are represented by a graph $G=(E, V)$ where vertices $V$ represent (incompatible) patient-donor pairs, and non-directed donors (NDDs) who are willing to donate without receiving a kidney in return. 
Directed edges $e\in E$ between vertices represent potential transplants from the donor of one vertex to the patient of another.
Edge weights represent the ``utility'' of an edge, and are typically set by exchange policy.
Solutions to a kidney exchange problem (henceforth, \emph{\matching{}s}) consist of both directed \emph{cycles} on $G$ containing only patient-donor pairs, and directed \emph{chains} beginning with an NDD and passing through one or more pairs; see Appendix~\ref{app:kpd} for an example exchange graph.
Each vertex may participate in only one edge in a \matching{}--as each vertex can donate and receive at most one kidney.

Vectors are denoted in bold, and are indexed by either cycles or edges: $\bm y_e$ indicates the element of $\bm y$ corresponding to edge $e$, and $\bm x_c$ is the element of $\bm x$ corresponding to cycle $c$.
Our notation uses a \emph{cycle-chain} representation for \matching{}s\footnote{Our experiments use the position-indexed formulation, which is more compact and equivalent~\cite{Dickerson16:Position}.}:
let $\mathcal C$ represent cycles and chains in $G$, where each cycle and chain corresponds to a list of edges; as is standard in modern exchanges, we assume that cycles and chains are limited in length.
\Matching{}s are expressed as a binary vector $\bm x \in \{0, 1\}^{|\mathcal C|}$, where $\bm x_c=1$ if cycle/chain $c$ is in the \matching{}, and $0$ otherwise.
Let $\bm w_c$ be the weight of cycle/chain $c$ (the sum of $c$'s edge weights).
%
%
Let $\mathcal M$ denote the set of \emph{feasible \matching{}s}--that is, the set of vertex-disjoint cycles and chains on $G$.
The total weight of a \matching{} is simply the summed weights of all its constituent cycles and chains: $\sum_{c\in \mathcal C} \bm x_c \bm w_c$.
We denote \emph{sets} of edges using binary vectors, where $\bm q\in \{0, 1\}^{|E|}$ represents the set of all edges with $\bm q_e=1$.

In the remainder of this paper we refer to pre-screening a transplant as \emph{querying an edge}, in order to be consistent with the literature.

%
%
%
%
%
%
%
\begin{figure}
    \centering
    \includegraphics[width=0.98\textwidth]{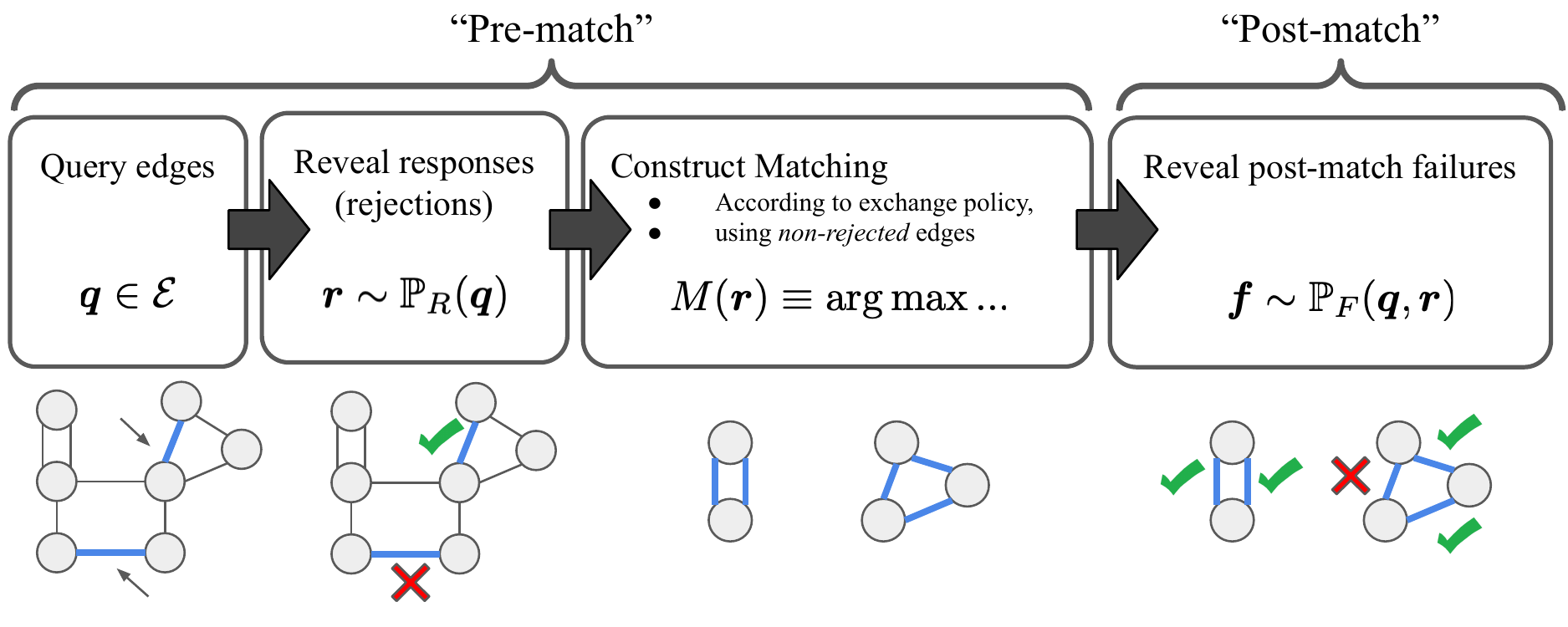}
    \caption{
    Single-stage edge selection:
    First, edges are selected to be queried, and responses revealed.
    Then, a final \matching{} is constructed according to the exchange's \matching{} policy.
    Finally, the post-match edge failures are revealed.
    }
    \label{fig:stages}
    \end{figure}


\noindent \textbf{Selecting Edge Queries.}
Our setting consists of two phases (see Figure~\ref{fig:stages}): during \emph{pre-match}, the decision-maker selects edges to query, and each queried edge is either accepted or rejected; then the decision-maker constructs a \matching{} using a fixed policy.
During \emph{post-match}, each match edge either fails (no transplant) or succeeds (the transplant proceeds).
We consider two version of the pre-match phase: in the \emph{single-stage} version, the decision-maker selects all queries before observing edge responses (accept/reject); in the \emph{multi-stage} version, one edge is selected at a time and responses are observed immediately.

Unlike most prior work, edges in our model may fail during both the pre- and post-match phase.  
For example, suppose the decision-maker queries an edge from a $60$-year-old non-directed donor, to a $35$-year-old recipient; if the recipient or their medical team rejects the elderly donor and decides to wait for a younger donor, this is a pre-match rejection.
Instead suppose the edge is not queried, and it is included in the final \matching{}; if medical screening reveals that the patient and donor are incompatible, this is a post-match failure.
We refer to pre-match failures as \emph{rejections} and post-match failures as \emph{failures}; however we make no assumption about their cause.
%
%
We represent potential failures and rejections using binary random variables: 
%
$\bm r \in \{0, 1\}^{|E|}$ denotes pre-match rejections, where $\bm r_e = 1$ if $e$ is queried and rejected, and $0$ otherwise ($\bm r_e=0$ for all non-queried edges).
Similarly $\bm f \in \{0, 1\}^{|E|}$ denotes post-match failures, where $\bm f_e = 1$ if edge $e$ fails post-match, and $0$ otherwise.
We assume that the distribution of rejections $\bm r\sim \prob_R(\bm q)$ is known, and depends on $\bm q$; we assume the distribution of failures $\bm f \sim \prob_F(\bm q, \bm r)$ is known, and depends on both $\bm q$ and $\bm r$.

Rejections and failures impact the matching through the \emph{weight} of each cycle and chain.
If any cycle edge fails, then \emph{no} transplants in the cycle can proceed; if a chain edge fails, than all edges \emph{following} it cannot proceed.\footnote{This assumes that chains can be \emph{partially} executed: for example, suppose that the $4^{th}$ edge in a $10$-edge chain fails; the first three edges can still be matched, and the post-failure chain weight sums only these three edges. Not all fielded exchanges use this policy: some exchanges cancel the entire chain if one of its edges fails.} 
Suppose we observe failures $\bm f$; the \emph{final \matching{} weight} of $c$ is 
$$
F(c, \bm y) \equiv \begin{cases}
\sum_{e\in c} w_e &\text{if}\, \sum_{e\in c} \bm y_e = 0 \\
0 &\text{if $c$ is a cycle and}\; \sum_{e\in c} \bm y_e > 0 \\
\sum_{e\in c'} w_e &\text{if $c$ is a chain, where $c'$ includes all edges up to the first failed edge.}
\end{cases}
$$
Thus the \emph{post-match expected weight} of \matching{} $\bm x$, due to both rejections $\bm r$ and failures $\bm f$, is
\begin{equation*}
W(\bm x; \bm q, \bm r) \equiv \expect_{\bm f \sim \prob_F(\bm q, \bm r)} \left[ \sum_{c\in \mathcal C} \bm x_c \, F(c, \bm r + \bm f) \right].
\end{equation*}
\noindent \textbf{\Matching{} Policy} 
In this paper we assume that the final \matching{} is constructed using a fixed \matching{} policy, which uses only \emph{non-rejected} edges; we denote this policy by $M(\bm r)$.
%
%
%
%
%
We focus primarily on the \emph{max-weight} policy $M^{\texttt{MAX}}(\cdot)$, which is used by most fielded exchanges, and the \emph{failure-aware} policy $M^{\texttt{FA}}(\cdot)$, which maximizes the expected post-match weight~\cite{dickerson2019failure}:
\begin{equation*}
M^{\texttt{MAX}}(\bm r) \in \argmax_{\bm x\in \mathcal M} \; \sum_{c\in \mathcal C} \bm x_c\, F(c, \bm r) \;, 
\hspace{0.2in} 
M^{\texttt{FA}}(\bm r) \in \argmax_{\bm x\in \mathcal M(\bm r)}\; \expect_{\bm f \sim \prob_F(\bm q, \bm r)}\;\left[ \sum_{c\in \mathcal C} \bm x_c \, F(c, \bm r + \bm f) \right]\;.
\end{equation*}
Evaluating this policy requires solving a kidney exchange clearing problem, which is NP-hard~\cite{Abraham07:Clearing}. 
However, state-of-the-art method can solve realistic kidney exchange clearing problems in fractions of a second (e.g., our experiments use the PICEF method of~\citet{Dickerson16:Position}); thus, throughout this paper we treat this policy as a low- or no-cost oracle. 

%
%
%
%
Next we formalize the \emph{edge selection problem}--the main focus of this paper.
%
%
We denote by $\mathcal E$ the set of ``legal'' edge subsets, subject to exchange-specific constraints; we assume that $\mathcal E$ is a matroid with ground set $E$.
For example, the decision-maker may limit the number of queries issued to any one medical team (vertex in $G$) or transplant center (group of vertices). 
We aim to select an edge set $\bm q\in  \mathcal E$ which maximizes the \emph{expected weight} of the final \matching{}.
These edges are selected using only the distribution of future rejections and failures; we take a \emph{stochastic optimization} approach, maximizing the expected outcome over this uncertainty.

\noindent \textbf{Single-Stage Setting.} 
The single-stage policy-constrained edge selection problem (henceforth, the \emph{edge selection problem}) is expressed as
\begin{equation}
    \max_{\bm q \in \mathcal E} \;V^S(\bm q)\, ,%
    \hspace{0.3in} 
    \text{with} 
    \quad V^S(\bm q) \equiv \expect_{\bm r \sim \prob_R(\bm q)} \big[  \; W\left(M(\bm r); \bm q, \bm r\right) \; \big] \;,\label{eq:singlestage} 
\end{equation}
where, $M(\bm r)$ denotes the \matching{} policy after observing rejections $\bm r$, and $W(\bm x; \bm q, \bm r)$ denotes the post-match expected weight of \matching{} $\bm x$.
Exact evaluation of $V^S(\bm q)$ is often intractable, as the support of $\prob_R(\bm q)$ grows exponentially in $|\bm q|$.
In experiments we approximate $V^S(\bm q)$ using sampling, and these approximations converge for a moderate number of samples (see Appendix~\ref{app:obj}).

\noindent \textbf{Multistage Setting.}
In the multi-stage setting, edge rejections are observed immediately after each edge is queried.
The multi-stage problem is expressed as
\begin{equation}
    \max_{\bm q^1 \in \mathcal E_1} \; \expect_{\bm r^1 \sim \prob_R(\bm q^1)} \left[ \; 
    \max_{\bm q^2 \in \mathcal E_1} \; \expect_{\bm r^2 \sim \prob_R(\bm q^2)} \left[ \; \dots \; \max_{\bm q^K \in \mathcal E_1}\; \expect_{\bm r^K \sim \prob_R(\bm q^K)} \; \left[ \; W\left(M(\bm r); \bm q, \bm r\right) \; \right]\; \right] \; \dots \;\right],\label{eq:multistage}
\end{equation}
where $\bm q \equiv \sum_{i=1}^K \bm q^i$ denotes all queried edges, $\bm r \equiv \sum_{i=1}^K \bm r^i$ denotes all rejections, and $\mathcal E_1\subseteq \mathcal E$ be denotes the legal edge subsets containing only one edge.
First, we observe that Problems~\ref{eq:singlestage} and~\ref{eq:multistage} require evaluating a \matching{} policy $M(\bm r)$.
In the case of kidney exchange, evaluating both the max-weight policy $M^{\texttt{MAX}}(\cdot)$ and the failure-aware policy $M^{\texttt{FA}}(\cdot)$ require solving NP-hard problems; thus Problems~\ref{eq:singlestage} and~\ref{eq:multistage} are at least NP-hard as well.
%

However, regardless of \matching{} policy, the question whether \emph{edge selection} is is hard.
We observe that while these problems are difficult in principle, experiments (\S~\ref{sec:expt}) show that they are easy in practice.
Proofs of the following propositions can be found in Appendix~\ref{app:proofs}.
%
%
\begin{proposition}\label{prop:non-monotonic}
With \matching{} policy $M^\texttt{FA}(\cdot)$, the objective of Problem~\ref{eq:singlestage} is non-monotonic in the number of queried edges, even with independent edge distributions.
\end{proposition}
In other words, querying additional edges can sometimes lead to a \emph{worse} outcome.
This is somewhat counter-intuitive; one might think that providing additional information to the \matching{} policy would strictly improve the outcome.
%
%
This is a worst-case result--and in fact our experiments demonstrate that querying edges almost always leads to a better final \matching{} weight.
%
%
\begin{proposition}\label{prop:non-submodular}
With \matching{} policy $M^\texttt{MAX}(\cdot)$, the objective of Problem~\ref{eq:singlestage} is non-submodular in the set of queried edges.
\end{proposition}
%
%
%
In other words, certain edges are \emph{complementary} to each other--and querying complementary edges simultaneously can yield a greater improvement than querying them separately.
%
Taken together, these propositions indicate that single-stage edge selection with \matching{} policy $M^\texttt{MAX}(\cdot)$ is a challenging combinatorial optimization problem.
On the other hand, using the failure-aware \matching{} policy $M^\texttt{FA}(\cdot)$ allows us to avoid some of these issues.
\begin{proposition}\label{prop:fail-aware-monotonic}
With \matching{} policy $M^\texttt{FA}(\cdot)$, and if all edges are independent, the objective of Problem~\ref{eq:singlestage} is monotonic in the set of queried edges.
\end{proposition}
While Propositions~\ref{prop:non-monotonic} and~\ref{prop:non-submodular} state that single-stage edge selection is challenging in the worst case, our computational results suggest that these problems are often easier on realistic exchanges.

\section{Solving the Policy-Constrained Edge Query Problem}\label{sec:solving-the-query-problem}

%
%
First we propose an exhaustive tree search which returns an optimal solution to Problem~\ref{eq:singlestage} given enough time.
Building on this, we propose a Monte Carlo Tree Search algorithm and a simple greedy algorithm.
Our multi-stage approaches are very similar to these, and can be found in Appendix~\ref{app:alg}.

%
Our optimal exhaustive search uses a \emph{search tree} where each tree node corresponds to an edge subset in $\bm q \in \mathcal E$.
The children of node $\bm q$ correspond to any $\bm q'\in \mathcal E$ which are equivalent to the parent $\bm q$, but include one additional edge: $C(\bm q) \equiv \{ (\bm q + \bm q') \; \forall \bm q' \in \mathcal E : \, |\bm q'| = 1\; \mid \; (\bm q + \bm q') \in \mathcal E \} \; $.
We say that edge sets (or tree nodes) containing $L$ edges are on the $L^{th}$ \emph{level} of the tree.
We refer to nodes with no children as \emph{leaf nodes}.
Unlike other tree search settings, the optimal solution to Problem~\ref{eq:singlestage} may be at \emph{any} node of the tree, not only leaf nodes; this is a consequence of non-monotonicity (see Proposition~\ref{prop:non-monotonic}).
The tree defined by root node $\bm q = \bm 0$ and child function $C(\bm q)$ contains all legal edge subsets in $\mathcal E$, when $\mathcal E$ is a matroid.
%
%
Thus, \emph{any} exhaustive tree search algorithm (such as depth-first search) will identify an optimal solution, given enough time and memory.

Of course exhaustive search is only tractable if $\mathcal E$ is small. 
Consider the class of \emph{budgeted} edge sets $\mathcal E(\Gamma)$ used in our experiments: $\mathcal E(\Gamma) \equiv \{\bm q \in \{0, 1\}^{|E|} \; \mid \; |\bm q|\leq \Gamma \}$ (edge sets containing at most $\Gamma$ edges).
The number of edge sets in $\mathcal E(\Gamma)$ grows roughly exponentially in $\Gamma$ and $|E|$, and is impossible to enumerate even for small graphs.
Suppose a graph has $50$ edges and we have an edge budget of five: there are over two million edge sets in $\mathcal E(5)$.   
Even small exchange graphs can have thousands of edges, and thus $\mathcal E(\Gamma)$ cannot be enumerated.
Therefore, we propose search-based approach.

\noindent \textbf{Monte Carlo Tree Search for Edge Selection (\MCTSS{}):}
We propose a tree-search algorithm for single-stage edge selection, \MCTSS{}, based on Monte Carlo Tree Search (MCTS), with the Upper Confidence for Trees (UCT) algorithm~\cite{kocsis2006bandit}.
Our approach keeps track of a \emph{value} (the objective value of Problem~\ref{eq:singlestage}) and a UCB value estimate for each node, and these values are updated during sampling.
The formula used to estimate a node's UCB value is
$$ \frac{\frac{U}{N} - V^{min}}{V^{max}- V^{min}} + \sqrt{N^P/N} $$
where $U$ is the ``UCB value estimate'' calculated by \MCTSS{}, $N$ is the number of visits to the node, $N^P$ is the number of visits to the node's parent, and $V^{max}$ and $V^{min}$ are the largest and smallest node values encountered during search.

%
When the set of tree nodes is too large to enumerate UCT can use a huge amount of memory--by storing values for each visited node.
To limit both memory use and runtime, we incrementally search the tree from a temporary root node.
Beginning from the root (the the empty edge set), we use UCB sampling on the next $L$ levels of nodes--where $L$ is a small fixed integer.
After a fixed time limit, sampling stops and we set the \emph{new} root node to the current root's best child according to its UCB estimate--using the method of~\cite{kocsis2006bandit}. 
This process repeats until we reach the final level of the search tree.
Algorithm~\ref{alg:uct-single-stage} gives a pseudocode description of \MCTSS{}, which uses Algorithm~\ref{alg:sample} as a submethod.
While often successful, MCTS requires extensive training and parameter tuning.
As a simpler alternative, we propose a greedy algorithm. %

\noindent \textbf{Single-Stage Greedy Algorithm: \GREEDYS{}.}
Like \MCTSS{}, our greedy algorithm (\GREEDYS{}) begins with the empty edge set as the root node, and iteratively searches deeper levels of the tree.
However unlike \MCTSS{}, \GREEDYS{} simply selects the child node with the greatest objective value in Problem~\ref{eq:singlestage}--that is, \emph{greedily} improving the objective value; see Appendix~\ref{app:alg} for a pseudocode description. 

\begin{minipage}[t]{0.46\textwidth}
\vspace{0pt} 
{\small
\begin{algorithm}[H]
\SetAlgoNoLine
\DontPrintSemicolon

(input) $K$: maximum size of any legal edge set\;
(input) $T$: time limit per level\;
(input) $L$: number of look-ahead levels\;
\;
$\bm q^R \gets \bm 0$ \quad root node (no edges)\;
$\bm q^* \gets \bm 0$ \quad the best visited node\;
$V^* \gets$ objective value of $\bm q^*$\;

\For{$N=1, \dots, K$}{
	$M \gets \min\{N + L, K \}$ \; 
	$Q \gets$ all nodes in levels $N$ to $M$\;
    $U[\bm q] \gets 0 \; \forall \bm q\in Q$  \quad UCB value estimate \;
    $V[\bm q] \gets 0 \;\forall \bm q\in Q$  \quad objective value \;
    $N[\bm q] \gets 0 \;\forall \bm q\in Q$ \quad number of visits \;
    \While{less than time $T$ has passed}{
    \texttt{Sample}($\bm q^R$, $M$)\;
    }
    $\bm q^R \gets \argmax_{\bm q \in C(\bm q^R)} U[\bm q]$\;
    Delete $U[\cdot]$, $V[\cdot]$, and $N[\cdot]$\;
}
\KwRet $\bm q^*$
\caption{\MCTSS{}: Tree Search for Single-Stage Edge Selection}
\label{alg:uct-single-stage}
\end{algorithm}
}
\end{minipage}
\hfill
\begin{minipage}[t]{0.52\textwidth}
\vspace{0pt} 
{\small
\begin{algorithm}[H]
\SetAlgoNoLine
\DontPrintSemicolon
(input) $\bm q$, $M$\;
\;
$N[\bm q]\gets N[\bm q] + 1$\;
$V[\bm q] \gets$ objective of edge set $\bm q$ in Problem~\ref{eq:singlestage} \;
\If{$V[\bm q] > V^*$}{
    $\bm q^* \gets \bm q$, \,$V^* \gets V[\bm q]$\;
}
\If{$\bm q$ has no children}{
\KwRet $V[\bm q]$}
\If{$\bm q$ has children}{
\eIf{$|\bm q|<M$}{
    $\bm q' \gets \argmax_{\bm q \in C(\bm q^R)} U[\bm q] + \text{UCB}[\bm q]$\;
    $U[\bm q] \gets U[\bm q] +$ \texttt{Sample}($\bm q'$, $M$)
}{
$\bm q' \gets $ a random descendent of $\bm q$ at any level\;
$V' \gets$ objective value of $\bm q'$ in Problem~\ref{eq:singlestage}\;
\If{$V' > V^*$}{
    $\bm q^* \gets \bm q'$, \,$V^* \gets V'$\;
}
$U[\bm q] \gets U[\bm q] + V'$
}
}
\caption{\texttt{Sample}: Sampling function used by \MCTSS{}}\label{alg:sample}
\end{algorithm}
}
\end{minipage}

\noindent \textbf{Runtime.}
Our methods rely on an ``oracle’’ to solve the NP-hard kidney exchange matching problem; while state-of-the-art methods solve real-sized instances of these problems in fractions of a second, there is no guaranteed bound for absolute runtime. 
Instead, we can report the \emph{number of calls} to this oracle for each method as a measure of complexity. 
Both benchmark methods (max-weight matching and failure-aware~\cite{dickerson2019failure}) as well as \texttt{IIAB}~\cite{blum2020ignorance} use exactly one oracle call; i.e., they are $O(1)$. 
Both \texttt{Greedy} and \texttt{MCTS} use a fixed number of samples ($M$) to evaluate the objective of an edge set. \texttt{Greedy} evaluates the objective of an edge set exactly $\Gamma$ times; thus, \texttt{Greedy} is $O(M\cdot \Gamma)$. 
Finally, \texttt{MCTS} can in theory visit all potential edge sets of size at most $\Gamma$ (i.e., an exhaustive search), which is $O(M\cdot\sum_{\gamma = 1}^\Gamma {|E| \choose \gamma})$.
Since this version of \texttt{MCTS} is intractable in both runtime and memory, Algorithm~\ref{alg:uct-single-stage} imposes reasonable limits on our implementation.

\section{Computational Experiments}\label{sec:expt}

We conduct a series of computational experiments using both synthetic data, and real kidney exchange data from UNOS; all code for these experiments is available online.\footnote{\url{https://github.com/duncanmcelfresh/kpd-edge-query}}
%
In these experiments, ``legal'' edge sets are the budgeted edge sets defined as $\mathcal E(\Gamma) \equiv \{\bm q\in \{0, 1\}^{|E|}\, \mid \, |\bm q|\leq \Gamma \}$.
In Sections~\ref{sec:expt-single-stage} and~\ref{sec:expt-multi} we present results in the single- and multi-stage edge selection settings, respectively.
We use two types of data for these experiments:

\noindent \textbf{Real Data.}
We use exchange graphs from the United Network for Organ Sharing (UNOS), representing UNOS match runs between $2010$ and $2019$.  
%
Some of these exchange graphs only have the trivial \matching{} (no cycles or chains), or they have only one non-trivial \matching{}.
We ignore these graphs because the \matching{} policy is a ``constant'' function (to return the one feasible \matching{}) and edge queries cannot change the outcome.
Removing these, we are left with $324$ UNOS exchange graphs.

\noindent \textbf{Synthetic Data.}
We generate random kidney exchange graphs based on directed Erd\H{o}s-R\'{e}nyi graphs defined using parameters $N$ and $p$: let $V$ be a fixed set of $N$ vertices; for each pair of vertices $(V_1, V_2)$ there is an edge from $V_1$ to $V_2$ with probability $p$, and an edge from $V_2$ to $V_1$ with probability $p$ (independent of the edge from $V_1$ to $V_2$).  
%
%
Any vertices with no incoming edges are considered NDDs. %

%
In these experiments edge rejections and failures are independently distributed for each edge $e$; let $P_R$ be the rejection probability, $P_Q$ is the post-match success probability if $e$ is queried/accepted, and $P_N$ is the success probability if $e$ is not queried.
To simulate edge rejections and failures we use two synthetic edge distributions: \emph{Simple} and \emph{KPD}.
In the \emph{Simple} distribution, $P_R=0.5$, $P_Q=1$, and $P_N=0.5$ for all edges.
The \emph{KPD} distribution is inspired by the fielded exchange setting from which we draw our real underlying compatibility graphs.
According to UNOS, about $34\%$ of all edges are rejected by a donor or recipient pre-match~\cite{Leishman19:Challenges}; we draw $P_R$ uniformly from $U(0.25, 0.43)$ for each edge.
Edges ending in highly-sensitized patients (who are often less healthy and more likely to be incompatible) are considered high-risk; for these edges we draw $P_Q$ from $U(0.2, 0.5)$ and $P_N$ from $U(0.0, 0.2)$.
For other edges we draw $P_Q$ from $U(0.9, 1.0)$ and $P_N$ from $U(0.8, 0.9)$.
%

\subsection{Single-Stage Edge Selection Experiments}\label{sec:expt-single-stage}

In this section we compare against the baseline of a max-weight \matching{} \emph{without} edge queries (using policy $M^\texttt{MAX}(\cdot)$).
Many fielded kidney exchanges use a variant of this matching policy, so by comparing against this baseline we are illustrating the impact of edge queries on the state-of-the-art matching policies used in many real exchanges.
Let $V_X$ be the objective%
\footnote{All objective values are estimated using up to $1000$ sampled rejection scenarios (see Appendix~\ref{app:obj}), as it is intractable to evaluate the exact objective of large edge sets.}
of Problem~\ref{eq:singlestage} achieved by method $X$, we calculate $\Delta^\texttt{MAX}$ (the relative difference from baseline) as
$\Delta^\texttt{MAX} \equiv (V_X - V^S(\bm 0))/V^S(\bm 0)$.
A value of $\Delta^\texttt{MAX} =0$ means that method $X$ did not improve over the baseline, a value of $\Delta^\texttt{MAX}=1$ means that $X$ achieved an objective $100\%$ greater than the baseline, and so on.
Furthermore a value of $\Delta^\texttt{MAX}>0$ means that method $X$ \emph{increases} the objective by querying edges, while $\Delta^\texttt{MAX}<0$ means that method $X$ decreases the objective by querying edges.

%
\noindent \textbf{Result: \GREEDYS{} is essentially Optimal with small random graphs.}
First we investigate the \emph{difficulty} of edge selection.
Using random graphs, we compare \GREEDYS{} to the \emph{optimal} solution to Problem~\ref{eq:singlestage}, found by exhaustive search (\OPT{}). 
We generate three sets of $100$ random graphs with $N=50$, $75$, and $100$ vertices, and each with $p=0.01$.
%
For all graphs we run both \OPT{} and \GREEDYS{} with edge budget $3$; we calculate the \emph{optimality gap} of \GREEDYS{} as $\texttt{\%OPT}\equiv 100\times (V_\OPT{} - V_\GREEDYS{}) / V_\OPT{}$, where $V_X$ denotes the objective achieved by method $X$.
($V_\OPT{} > 0$ in all graphs used in these experiments.)
If $\texttt{\%OPT}=0$ then \GREEDYS{} returns an optimal solution, and $\texttt{\%OPT}>0$ means that \GREEDYS{} is not optimal.
Table~\ref{tab:grd-opt-and-iiab} (left) shows the number of random graphs binned by $\%\texttt{OPT}$, as well as the maximum $\%\texttt{OPT}$ over all graphs.
For each $N$, \GREEDYS{} returns an optimal solution for at least $90$ of the $100$ graphs; the \emph{maximum} $\%\texttt{OPT}$ over all graphs is $2.8$.

In other words, \GREEDYS{} always returns an \emph{optimal} or nearly-optimal set of edges to query for small random graphs.
This is somewhat unexpected, since the edge selection problem is both non-monotone and non-submodular (see Section~\ref{sec:edge-selection-problem}).

 \begin{table}
 \caption{\footnotesize Left: Optimality gap for \GREEDYS{}, over $100$ random graphs with $p=0.01$ and various $N$, with edge budget $\Gamma=3$; bottom row shows the maximum value of $\%\texttt{OPT}$ over all graphs. Right: Single-stage results on UNOS graphs using the variable IIAB edge budget (top rows), and the failure-aware method (bottom row). Columns $P_{X}$ indicates the $X^{th}$ percentile of $\Delta^\texttt{MAX}$ over all UNOS graphs.\label{tab:grd-opt-and-iiab}}
   \begin{minipage}[t]{0.4\textwidth}
   \footnotesize
   \centering
    \vspace{0.05in}
    \renewcommand*{\arraystretch}{1.2} 
    \begin{tabular}[t]{@{}lccc@{}}\toprule
    & \multicolumn{3}{c}{Num. Graphs (out of 100)} \\ \cmidrule{2-4}
    $\%\texttt{OPT}$ &  {\scriptsize $N=50$} &  {\scriptsize $N=75$} &  {\scriptsize $N=100$} \\
     \midrule
    $[0, 0.1]$ & $93$ & $93$ & $90$\\ 
    $(0, 1]$ & $5$ & $4$ & $9$\\ 
    $(1, 2]$ & $1$ & $3$ & $1$\\ 
    $(2, 100]$ & $1$ & $0$ & $0$\\ \midrule 
    Max $\%\texttt{OPT}$ & 2.8 & 1.5 & 1.0 \\\bottomrule
    \end{tabular}
    \end{minipage}
    \hfill
 \begin{minipage}[t]{0.55\textwidth}
    \footnotesize
    \centering
        \vspace{0.05in}
        \renewcommand*{\arraystretch}{1.2}    \setlength{\tabcolsep}{4pt}
    \begin{tabular}[t]{@{}lrrrcrrr@{}}
    \toprule
     & \multicolumn{3}{c}{\emph{Simple} edge dist.} && \multicolumn{3}{c}{\emph{KPD} edge dist.}  \\
    \cmidrule{2-4} \cmidrule{6-8}
    \textbf{Method} & $P_{10}$  & $P_{50}$  & $P_{90}$ && $P_{10}$ & $P_{50}$ & $P_{90}$ \\ \midrule
    \MCTSS{} & $0.40$  & $0.67$ & $1.11$ && $0.05$                  & $0.45$ & $3.44$ \\
    \GREEDYS{}        & $0.47$                     & $0.64$ & $1.00$ && $0.02$                  & $0.47$ & $3.44$ \\
    Random        & $0.00$                     & $0.10$ & $0.46$ && $-0.11$                  & $0.00$ & $0.63$ \\
    \texttt{IIAB}          & $0.21$                     & $0.45$ & $0.89$  && $-0.27$                 & $0.12$ & $2.24$ \\ \midrule
    Fail-Aware & $0.00$                     & $0.09$ & $0.23$ && $-0.27$\textsuperscript{$\dagger$} & $0.00$\textsuperscript{$\dagger$} & $2.17$\textsuperscript{$\dagger$} \\ \bottomrule
    \end{tabular}
    \end{minipage}
    \hfill
\end{table}
\customfootnotetext{$\dagger$}{We use an approximation of Fail-Aware for the \emph{KPD} dist.; \emph{true} Fail-Aware should always have $\Delta^\texttt{MAX}>0$.}

\noindent \textbf{Result: \GREEDYS{} is essentially monotonic with UNOS graphs.} 
We test \GREEDYS{} on real UNOS graphs, using maximum budget $100$.
Figure~\ref{fig:grd-large-budget} shows the median $\Delta^\texttt{MAX}$ over all UNOS graphs, with shading between the $10^{th}$ and $90^{th}$ percentiles.
Larger edge budgets almost never decrease the objective achieved by \GREEDYS{}, and \GREEDYS{} \emph{never} produces a worse outcome than the baseline.
Thus--in our setting--single-stage edge selection is effectively monotonic in our setting, and \GREEDYS{} is an effective method.

\noindent \textbf{Result: \MCTSS{} and \GREEDYS{} are nearly equivalent with UNOS graphs.}
%
%
We compare all methods on UNOS graphs, using smaller, more-realistic edge budgets from $1$ to $10$.
For \MCTSS{} we use a 1-hour time limit per edge ($\Gamma$ hours total).
Figures~\ref{fig:unos-simple} and~\ref{fig:unos-kpd} compare $\Delta^\texttt{MAX}$ for \MCTSS{}, \GREEDYS{}, and random edge selection, for the \emph{Simple} and \emph{KPD} edge distributions, respectively.
We draw two conclusions from these results: (1) \MCTSS{} and \GREEDYS{} produce almost identical results, further suggesting that \GREEDYS{} is  nearly optimal in our setting; (2) in our setting, edge selection is \emph{effectively} monotonic, as $\Delta^\texttt{MAX}$ almost never decreases.
However Figure~\ref{fig:unos-kpd} gives an example of non-monotonicity for both \GREEDYS{} and Random: in some cases, querying edges can lead to a \emph{worse} outcome than querying no edges.

\noindent \textbf{Result: Both \MCTSS{} and \GREEDYS{} outperform benchmarks from the literature.}
We also compare against two state-of-the-art approaches: the edge selection approach of~\cite{blum2020ignorance} (\texttt{IIAB}), which uses a \emph{variable} edge budget that depends on the graph structure; and and the failure-aware \matching{} policy of~\cite{dickerson2019failure} (Fail-Aware\footnote{For the \emph{KPD} distribution we use an approximation of Fail-Aware, which assumes a uniform edge failure probability.}), which does not query edges
To our knowledge, \texttt{IIAB} is the only edge selection method in the literature.
We compare against the Fail-Aware method because it is a state-of-the-art kidney exchange matching policy which aims to maximize the expected matching weight, under a similar edge failure model to ours; we compare against this approach to further illustrate the utility of querying edges.

%
%
Table~\ref{tab:grd-opt-and-iiab} (right) shows a comparison of all edge-selection methods--each using the variable edge budget of \texttt{IIAB}; the bottom row shows results for Fail-Aware.
Both \MCTSS{} and \GREEDYS{} achieve greater $\Delta^\texttt{MAX}$ (in distribution) than both benchmark methods.
This is expected in both cases: \texttt{IIAB} uses a heuristic to select edges to query, which does not consider the final matching weight---the objective of our edge selection problem; on the other hand, both \MCTSS{} and \GREEDYS{} are designed to maximize this objective.
We do not expect Fail-Aware to out-perform any edge selection methods, since Fail-Aware does not have access to information revealed after edge queries.

It is notable that \GREEDYS{} performs better than \MCTSS{} (in distribution).
This likely means that \MCTSS{} is \emph{under-trained}---that the time and memory limits used in our implementation are too restrictive; alternatively, this indicates that \GREEDYS{} is simply very effective in our setting.

\begin{figure}	
    %
	\begin{subfigure}[t]{0.49\textwidth}
    	\centering
    	\includegraphics[width=\linewidth]{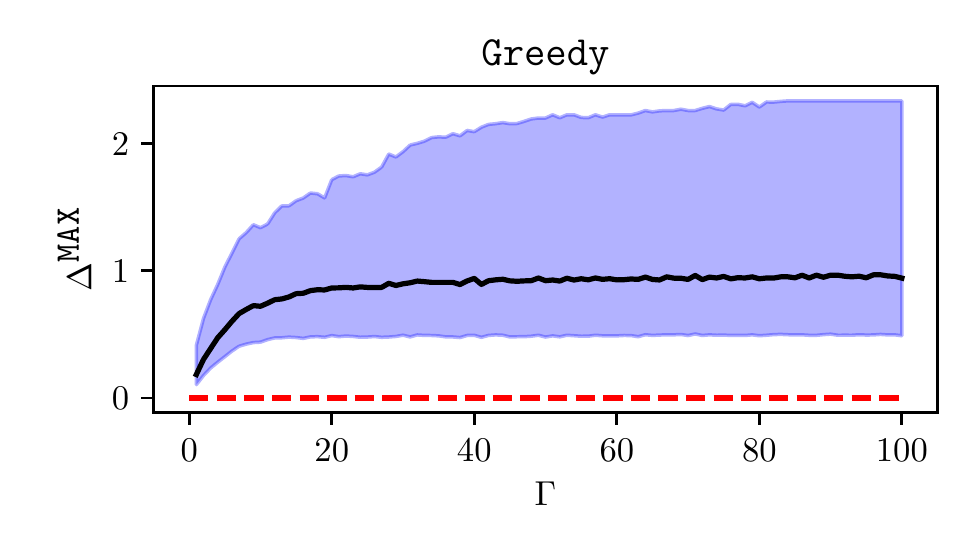}
    	\caption{Single-stage, \emph{Simple} distribution.\label{fig:grd-large-budget}}
    \end{subfigure}
\hfill
	\begin{subfigure}[t]{0.49\linewidth}
		\centering
		\includegraphics[width=\linewidth]{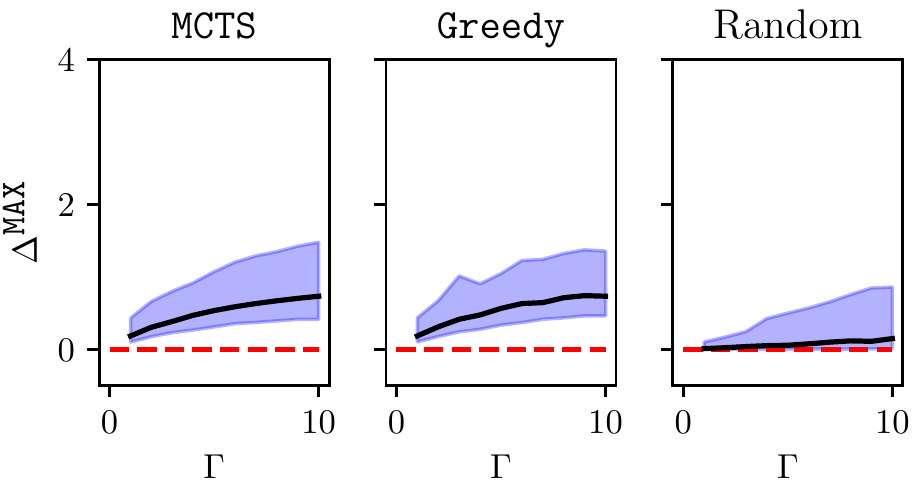}
		\caption{Single-stage, \emph{Simple} distribution.}\label{fig:unos-simple}
	\end{subfigure}
	\hfill
	%
	\begin{subfigure}[t]{0.49\textwidth}
		\centering
		\includegraphics[width=\linewidth]{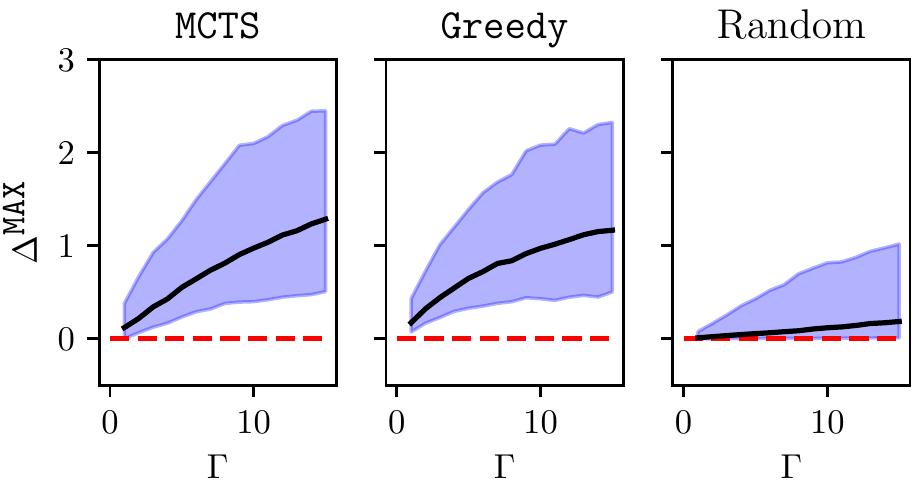}
		\caption{Multi-stage, \emph{Simple} distribution.\label{fig:unos-multistage}}
	\end{subfigure}
	\hfill
	%
	\begin{subfigure}[t]{0.49\textwidth}
	\centering
	\includegraphics[width=\linewidth]{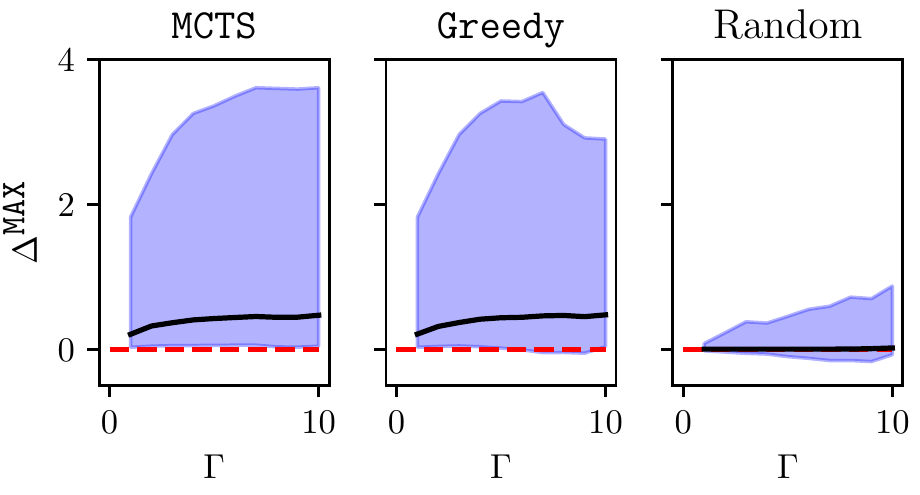}
	\caption{Single-stage, \emph{KPD} distribution.}\label{fig:unos-kpd}
	\end{subfigure}
	\caption{Results for UNOS graphs. 
	Right: edge budget up to $10$ for the \emph{Simple} distribution (top) and the \emph{KPD} distribution (bottom). 
	Top-left: \GREEDYS{} with edge budget up to $100$, for the simple distribution. 
	Bottom-left: multi-stage methods using the \emph{Simple} distribution.
	In all plots, a solid line indicates median $\Delta^\texttt{MAX}$ over all UNOS graphs, and shading is between the $10^{th}$ and $90^{th}$ percentiles; a dotted line indicates the baseline.
	}\label{fig:unos}
\end{figure}

\subsection{Multi-Stage Edge Selection Experiments on UNOS Graphs}\label{sec:expt-multi}
We run initial multi-stage edge selection experiments on all UNOS graphs with the \emph{Simple} edge distribution.
For each graph we test our multi-stage variants of \MCTSS{} and \GREEDYS{}, and compare with a baseline of random edge selection; as before, \MCTSS{} uses a 1-hour training time per level. 
It is substantially harder to evaluate the multi-stage objective, as each edge edge-selection method changes depending on rejections observed in prior stages.
Similarly, the \MCTSS{} search tree is orders of magnitude larger in the multi-stage setting: each node in tree corresponds to both an edge set \emph{and} a rejection scenario (see Appendix~\ref{app:alg}).

In these initial experiments we evaluate each method on $10$ edge rejections \emph{realizations} (only a small subset).
We estimate $\Delta^\texttt{MAX}$ for each method and each graph by averaging the final \matching{} weight over all realizations.
Figure~\ref{fig:unos-multistage} shows the results of these experiments.

These initial multi-stage results are quite similar to our single-stage results.
However it is notable that the objective value in the multi-stage setting is somewhat higher than in the single-stage setting--even using the simple method \GREEDYS{}.
Further, this suggests that more can be gained by developing a more sophisticated multi-stage edge selection policy. 
We leave this for future work.

\section{Conclusions and Future Research Directions}
Many planned kidney exchange transplants \emph{fail} for a variety of reasons; these failures greatly reduce the number of transplants that an exchange can facilitate, and increase the waiting time for many patients in need of a kidney.
Avoiding transplant failures is a challenge, as exchanges are often constrained by policy and law in how they match patients and donors.
We consider a setting where exchanges can \emph{pre-screen} certain transplants, while still matching patients and donors using a fixed policy.
We formalize a multi-stage optimization problem based on realistic assumptions about how transplants fail, and how exchanges match patients and donors; we emphasize that these important assumptions are not included in prior work.
While this problem is challenging in theory, we show that it is much easier in practice--with computational experiments using both synthetic data and real data from the United Network for Organ Sharing.
In experiments, we find that pre-screening even a small number of potential transplants (around $10$) significantly increases the overall quality of the final match--by more than $100\%$ of the original match weight.

Our initial study of the pre-screening problem suggests several areas for future work.
First we assume that the distribution of transplant failures is known, when in reality only rough approximations of these distributions are available.
Second, we assume that exchange participants (donors, recipients, hospitals) are not strategic.
In reality, strategic behavior plays a substantial role in real exchanges~\cite{agarwal2019market}; we expect that participants might behave strategically when responding to pre-screening requests.
Third, our model does not account for equitable treatment of different patients~\cite{McElfresh18:fair}.
For example, it may be the case that pre-screening a transplant decreases the likelihood of the transplant being matched.
That might disproportionately impact highly-sensitized patients, which are both sicker and more difficult to match than other patients. 



\clearpage
\section*{Broader Impact}
This work lives within the broader context of kidney exchange research. 
For clarity, we separate our broader impacts into two sections: first we discuss the impact of kidney exchange in general; then we discuss our work in particular, within the context of kidney exchange research and practice.

\paragraph{Impacts of Kidney Exchange}
Patients with end-stage renal disease have only two options: receive a transplant, or undergo dialysis once every few days, for the rest of their lives. 
In many countries (including the US), these patients register for a deceased donor waiting list--and it can be months or years before they receive a transplant.
Many of these patients have a friend or relative willing to donate a kidney, however many patients are incompatible with their corresponding donor.
Kidney exchange allows patients to ``swap'' their incompatible donor, in order to find a \emph{higher-quality} match, \emph{more quickly} than a waiting list.
Transplants allow patients a higher quality of life, and cost far less, than lifelong dialysis.
About $10\%$ of kidney transplants in the US are facilitated by an exchange.

Finding the ``most efficient'' matching of kidney donors to patients is a (computationally) hard problem, which cannot be solved by hand in most cases. 
For this reason many fielded exchanges use algorithms to quickly find an efficient matching of patients and donors. 
Many researchers study kidney exchange from an algorithmic perspective, often with the goal of improving the number or quality of transplants facilitated by exchanges. 
Indeed, this is the purpose of our paper.

\paragraph{Impacts of Our Work}
In this paper we investigate the impact of pre-screening certain potential transplants (edge) in an exchange, prior to constructing the final patient-donor matching.
To our knowledge, some modern fielded exchanges pre-screen potential transplants in an ad-hoc manner; meaning they do not consider the impacts of pre-screening on the final matching.
We propose methods to estimate the importance of pre-screening each edge, as measured by the change in the overall number and quality of matched transplants.\footnote{Quality and quantity of transplants is measured by transplant weight, a numerical representation of transplant quality (e.g., see UNOS/OPTN Policy 13 regarding KPD prioritization points \texttt{\url{https://optn.transplant.hrsa.gov/media/1200/optn_policies.pdf}}).}
Importantly, our methods do not require a change in matching policy; instead, they indicate to policymakers which potential transplants are important to pre-screen, and which are not.  
The impacts of our contributions are summarized below:

\noindent \textbf{Some potential transplants cannot be matched}, because they cannot participate in a ``legal'' cyclical or chain-like swap (according to the exchange matching policy).
    Accordingly, there is no ``value'' gained by pre-screening these transplants; our methods will identify these potential transplants, and will recommend that they not be pre-screened.
    Pre-screening requires doctors to spend valuable time reviewing potential donors; removing these unmatchable transplants from pre-screening will allow doctors to focus only on transplants that are relevant to the current exchange pool.
    
\noindent \textbf{Some transplants are more important to pre-screen than others}, and our methods help identify which are most important for the final matching. 
    We estimate the value pre-screening of each transplant by \emph{simulating} the exchange matching policy in the case that the pre-screened edge is pre-accepted, and in the case that it is pre-refused.
    
\noindent \textbf{To estimate the value of pre-screening each transplant, we need to know (a) the likelihood that each transplant is pre-accepted and pre-refused, and (b) the likelihood that each planned transplant fails for any reason, after being matched.} 
    These likelihoods are used as input to our methods, and they can influence the estimated value of pre-screening different transplants. 
    Importantly, it may not be desirable to calculate these likelihoods for each potential transplant (e.g., using data from the past). 
    For example if a patient is especially sick, we may estimate that any potential transplant involving this patient is very likely to fail prior to transplantation (e.g., because the patient is to ill to undergo an operation). 
    In this case, our methods may estimate that all potential transplants involving this patient have very low ``value'', and therefore recommend that these transplants should not be pre-screened. 
    One way to avoid this issue is to use the same likelihood estimates for all transplants.
    
\noindent \textbf{To estimate the impact of our methods (and how they depend on the assumed likelihoods, see above), we recommend using extensive modeling of different pre-screening scenarios before deploying our methods in a fielded exchange.}
    This is important for several reasons: first, exchange programs cannot always \emph{require} that doctors pre-screen potential transplants prior to matching. 
    Since we cannot be sure which transplants will be pre-screened and which will not, simulations should be run to evaluate each possible scenario.
    Second, theoretical analysis shows that pre-screening transplants can---in the worst case---negatively impact the final outcome.
    While this worst-case outcome is possible, our computational experiments show that it is very unlikely; this can be addressed further with mode experiments tailored to a particular exchange program. 

\section*{Acknowledgments}
We thank Ruthanne Leishman and Morgan Stuart at UNOS for very helpful early-stage discussions, feedback on our general approach, clarifications regarding data, and knowledge of the intricacies of running a fielded exchange.  Curry, Dickerson, and McElfresh were supported in part by NSF CAREER Award IIS-1846237, DARPA GARD, DARPA SI3-CMD \#S4761, DoD WHS Award \#HQ003420F0035, NIH R01 Award NLM-013039-01, and a Google Faculty Research Award.  Sandholm was supported in part by the National Science Foundation under grants IIS-1718457, IIS-1617590, IIS-1901403, and CCF-1733556, and the ARO under award W911NF-17-1-0082.

\bibliographystyle{abbrvnat}
\bibliography{refs, bib}

\clearpage
\appendix
\section{Kidney Exchange and Edge Failures}\label{app:kpd}

\noindent\textbf{Brief history.} \citet{Rapaport86:Case} proposed the initial idea for kidney exchange, while the first organized kidney exchange, the New England Paired Kidney Exchange (NEPKE), started in 2003--04~\citep{Roth04:Kidney,Roth05:Kidney,Roth05:Pairwise}.  NEPKE has since ceased to operate; at the point of cessation, its pool of patients and donors was merged into the United Network for Organ Sharing (UNOS) exchange in late 2010.  That exchange now contains over 60\% of transplant centers in the US, and performs matching runs via a purely algorithmic approach (as we discuss in Sections~\ref{sec:intro} and~\ref{sec:edge-selection-problem}, and in much greater depth by~\citet{UNOS}, which is mandated to transparently and publicly reveal its matching process).

There are also two large private kidney exchanges in the US, the National Kidney Registry (NKR) and the Alliance for Paired Donation (APD).  They typically only work with large transplant centers.  NKR makes their matching decisions manually and greatly prefers matching incrementally through chains.  APD makes their decisions through a combination of algorithmic and manual decision making. There are also several smaller private kidney exchanges in the US.  They typically only involve one or a couple of transplant centers.  These include an exchange at Johns Hopkins University, a single-center exchange
at the Methodist Specialty and Transplant Hospital in San Antonio, and a single-center exchange at Barnes-Jewish Hospital affiliated with the Washington University in St.\ Louis.  Largely, these exchanges also make their matching decisions via a combined algorithmic and manual process.   These exchanges compete in a variety of ways (e.g., by allowing patient-donor pairs to register in multiple exchange programs); this competition can lead to loss in efficiency~\cite{agarwal2019market} as well as sub-optimal changes to individual exchanges' matching polices~\cite{Li18:Equilibrium}.

There are now established kidney exchanges in the UK~\cite{Manlove15:Paired}, Italy, Germany, Netherlands, Canada, England, Portugal, Israel, and many other countries.  European countries are also explicitly exploring connecting their individual exchanges together in various ways~\cite{biro2019building}.

\noindent\textbf{Edge failures.}
The dilemma of edge failures is illustrated in the example exchange graph shown in Figure~\ref{fig:exchange}.
This exchange consists of a $3$-chain (dashed edges) and two $2$-cycles (solid edges).
Suppose the decision-maker queries edge $e_A$: if $e_A$ is accepted, then the chain from the NDD ($n$) through pairs $(d_1, p_1)$, $(d_2, p_2)$, and $(d_3, p_3)$, i.e., the dashed edges, can be included in the matching.
However if $e_A$ is queried and rejected, then the NDD cannot initiate the chain, and only the cycles may be matched.
In our model, if $e_A$ is not queried then it may still be matched.


\tikzstyle{altruist}=[circle,
  thick,
  minimum size=1.0cm,
  draw=black!50!green!80,
  fill=black!20!green!20
]

\tikzstyle{every path}=[line width=1pt]

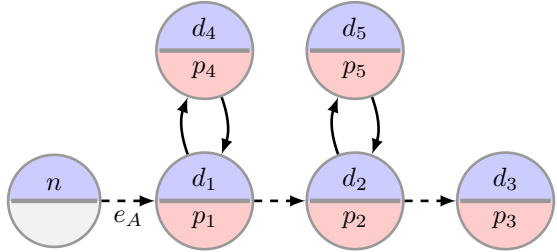
\begin{figure}[h]
\centering
\begin{tikzpicture}[>=latex]

  \tikzstyle{fake}=[rectangle,minimum size=8mm,opacity=0.0]
  
  \node (alt) [shape=circle split,
    draw=gray!80,
    line width=0.6mm,text=black,font=\bfseries,
    circle split part fill={blue!20,gray!10},
    minimum size=1.2cm,
  ] at (0,0) {$n$\nodepart{lower}};

  \node (p_1) [shape=circle split,
    draw=gray!80,
    line width=0.6mm,text=black,font=\bfseries,
    circle split part fill={blue!20,red!20},
    minimum size=1.2cm,
  ] at (2,0) {$d_1$\nodepart{lower}$p_1$};

    \node (p_2) [shape=circle split,
    draw=gray!80,
    line width=0.6mm,text=black,font=\bfseries,
    circle split part fill={blue!20,red!20},
    minimum size=1.2cm,
  ] at (4,0) {$d_2$\nodepart{lower}$p_2$};

    \node (p_3) [shape=circle split,
    draw=gray!80,
    line width=0.6mm,text=black,font=\bfseries,
    circle split part fill={blue!20,red!20},
    minimum size=1.2cm,
  ] at (6,0) {$d_3$\nodepart{lower}$p_3$};
  
     \node (p_4) [shape=circle split,
    draw=gray!80,
    line width=0.6mm,text=black,font=\bfseries,
    circle split part fill={blue!20,red!20},
    minimum size=1.2cm,
  ] at (2,+2) {$d_4$\nodepart{lower}$p_4$};
  
     \node (p_5) [shape=circle split,
    draw=gray!80,
    line width=0.6mm,text=black,font=\bfseries,
    circle split part fill={blue!20,red!20},
    minimum size=1.2cm,
  ] at (4,+2) {$d_5$\nodepart{lower}$p_5$};
  


  \draw[dashed,->] (alt) -- (p_1) node[midway,below,rotate=0] {$e_A$};
  \draw[dashed,->] (p_1) -- (p_2);
  \draw[dashed,->] (p_2) -- (p_3);
 
  \draw[->] (p_1) edge [bend left=20] (p_4);
  \draw[->] (p_4) edge [bend left=20] (p_1);
  \draw[->] (p_2) edge [bend left=20] node[midway,left,rotate=0] {} (p_5);
  \draw[->] (p_5) edge [bend left=20] node[midway,right,rotate=0] {} (p_2);

\end{tikzpicture} 
\caption{
Sample exchange graph with a $3$-chain (dashed edges) and two $2$-cycles (solid edges). 
The NDD is denoted by $n$, and each patient (and associated donor) is denoted by $p_i$ ($d_i$).
If edge $e_1$ is not queried, or queried and \emph{accepted}, then the chain may be included in the final matching.
However if edge $e_A$ is queried and \emph{rejected}, then only the $2$-cycles may be included in the final matching.
}\label{fig:exchange}
\end{figure}

\section{Estimating The Objective of Problem~\ref{eq:singlestage}}\label{app:obj}

The objective of the single-stage edge selection problem requires evaluating all rejection scenarios $\bm r\sim \prob_R(\bm q)$, and the support of this distribution grows exponentially in the number of edges $|\bm q|$. 
In computational experiments, to estimate the objective of Problem~\ref{eq:singlestage}, we sample up to $1000$ scenarios from $\prob_R(\bm q)$.
More explicitly: we \emph{exactly} evaluate the objective of edge sets with fewer than $10$ edges; for larger edge sets, we sample the objective using $1000$ draws from $\prob_R(\bm q)$.

Using bootstrapping experiments we demonstrate that our sampling approach is sufficient to accurately estimate the true objective, even for large edge sets.
For $152$ UNOS graphs, we computed edge sets by running \GREEDYS{} with edge budgets ranging from 1 to 100.
%
%
For each edge set, we then sample a subset of $N\in \{10, 30, 50, 100, 1000\}$ rejection scenarios, with replacement, from the set of all sampled edge outcomes. For each edge set and choice of $N$ we repeat 200 times and calculate the sample mean for each replication.
We then compute the standard deviations of these bootstrap sample means to estimate the variance due to sampling.
For each $N$, we calculate the mean sample standard deviation, normalized by the sample mean.
Table~\ref{tab:bootstrap} shows the median normalized standard deviation for all experiments under each $N$, with edge budgets aggregated into 10 bins.
We find that with $N=1000$ samples, the standard deviation was on average only about 2\% of the overall mean value, even for large edge budgets.

%
%
%
%

\begin{table}
\centering
\begin{tabular}{@{}c|ccccc@{}}\toprule
Edge budgets & $N=10$ & $N=30$ & $N=50$ & $N=100$ & $N=1000$ \\
\midrule
1-10 & 0.10 & 0.06 & 0.04 & 0.03 & 0.01 \\
11-20 & 0.12 & 0.07 & 0.05 & 0.04 & 0.01 \\
21-30 & 0.13 & 0.08 & 0.06 & 0.04 & 0.01 \\
31-40 & 0.14 & 0.08 & 0.06 & 0.04 & 0.01 \\
41-50 & 0.14 & 0.08 & 0.06 & 0.04 & 0.01 \\
51-60 & 0.15 & 0.08 & 0.07 & 0.05 & 0.01 \\
61-70 & 0.15 & 0.09 & 0.07 & 0.05 & 0.02 \\
71-80 & 0.16 & 0.09 & 0.07 & 0.05 & 0.02 \\
81-90 & 0.17 & 0.10 & 0.08 & 0.05 & 0.02\\
91-100 & 0.18 & 0.10 & 0.08 & 0.06 & 0.02\\

\bottomrule
\end{tabular}
\caption{Median normalized standard deviation of the bootstrap mean, over $200$ bootstrap samples for each sample size $N$, binned by edge budget.}
\label{tab:bootstrap}
\end{table}



\section{Additional Computational Results}\label{app:results}

First we show results for both single-stage and multi-stage edge selection on random graphs (see \S~\ref{sec:expt} for a description of these graphs).
For $N=50$, $75$, and $100$, we generate $30$ random graphs with $N$ vertices and $p=0.01$.
For each graph we run single-stage experiments with $\Gamma=1, \dots, 10$ and multi-stage experiments with $\Gamma=1, \dots, 15$.
Unlike experiments on UNOS graphs we use a time limit of $20$ minutes per edge; all other parameters are the same.
Figure~\ref{fig:random-single} and~\ref{fig:random-multi} show single-stage and multi-stage results for all random graphs, respectively.
Table~\ref{tab:random-iiab} shows comparisons to \texttt{IIAB} and Fail-Aware for random graphs with $N=50$, $75$, and $100$.

\begin{figure}	
    %
	\begin{subfigure}[t]{0.49\textwidth}
    	\centering
    	\includegraphics[width=\linewidth]{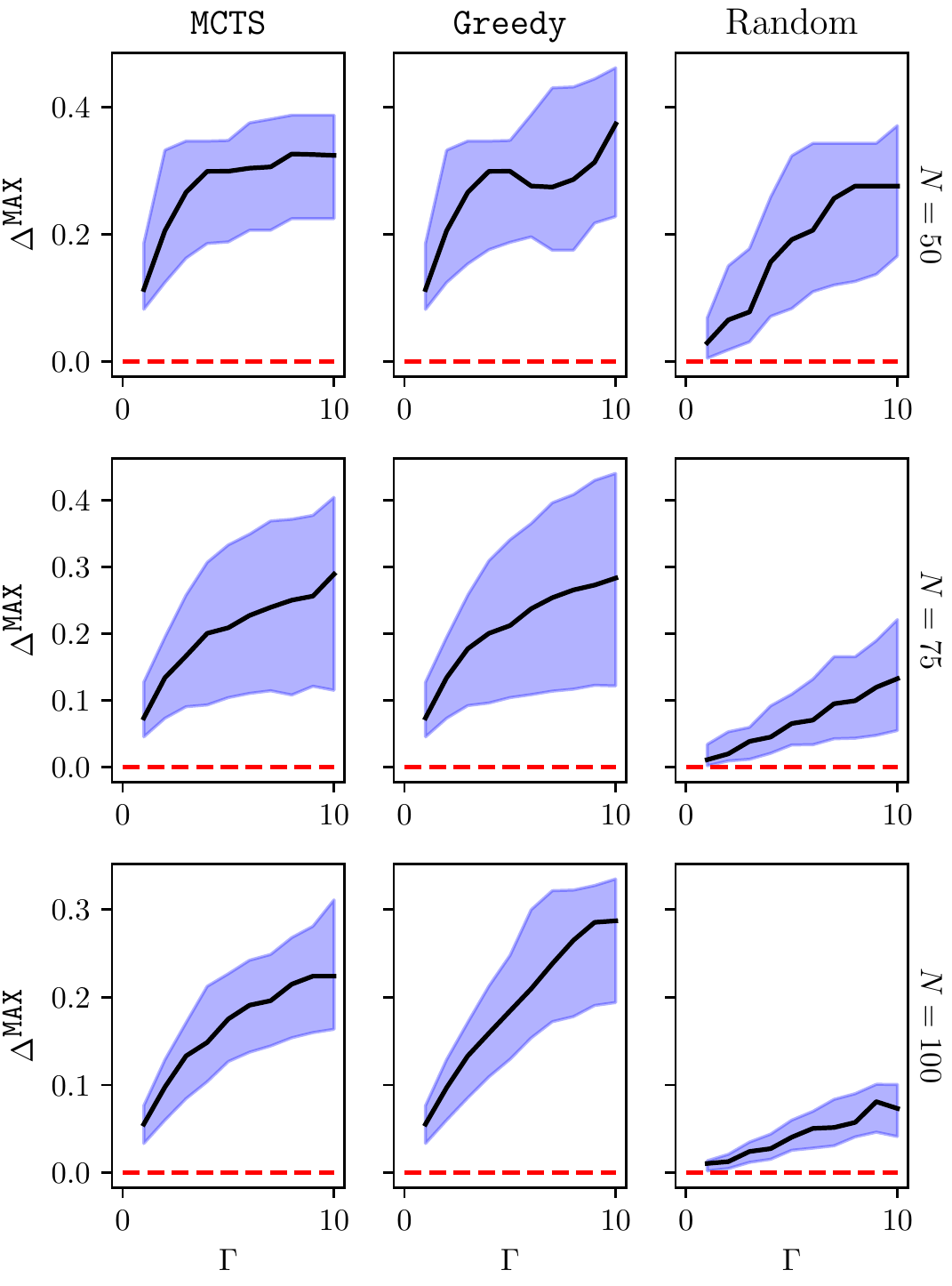}
    	\caption{Single-stage.\label{fig:random-single}}
    \end{subfigure}
\hfill
	\begin{subfigure}[t]{0.49\linewidth}
		\centering
		\includegraphics[width=\linewidth]{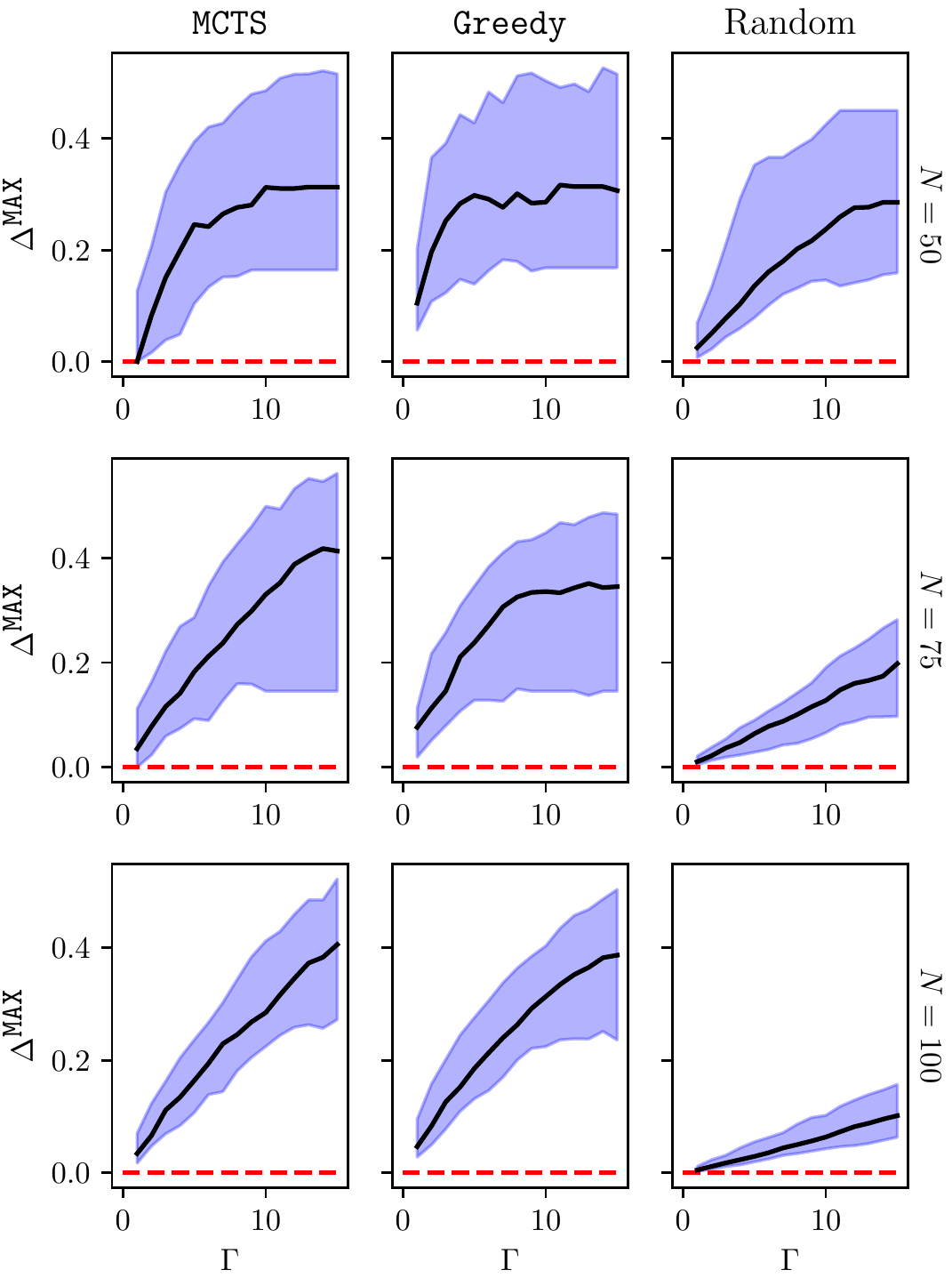}
		\caption{Multi-stage.}\label{fig:random-multi}
	\end{subfigure}
	\caption{Results for $30$ random graphs with edge probability $p=0.01$ and $N=50$ vertices (top row), $N=75$ (middle row), and $N=100$ (bottom row). 
	All experiments use the \emph{Simple} edge distribution.
	In all plots, a solid line indicates median $\Delta^\texttt{MAX}$ over all $30$ random graphs, and shading is between the $10^{th}$ and $90^{th}$ percentiles; a dotted line indicates the baseline.
	}
\end{figure}

\begin{table}[]
\caption{Single-stage results on random graphs with the \emph{Simple} edge distribution, using the variable IIAB edge budget (top rows), and the failure-aware method (bottom row). Columns $P_{X}$ indicates the $X^{th}$ percentile of $\Delta^\texttt{MAX}$ over all $30$ random graphs, for graphs with $N=50$, $75$, and $100$ vertices.\label{tab:random-iiab}}
   \centering
    \renewcommand*{\arraystretch}{1.2}    \setlength{\tabcolsep}{4pt}
    \begin{tabular}[t]{@{}lrrrcrrrcrrr@{}}
    \toprule
     & \multicolumn{3}{c}{$N=50$} && \multicolumn{3}{c}{$N=75$} &&
     \multicolumn{3}{c}{$N=100$} \\
    \cmidrule{2-4} \cmidrule{6-8} \cmidrule{10-12}
    \textbf{Method} & $P_{10}$  & $P_{50}$  & $P_{90}$ && $P_{10}$ & $P_{50}$ & $P_{90}$ && $P_{10}$  & $P_{50}$  & $P_{90}$ \\ \midrule
    \MCTSS{} & $0.22$ & $0.30$ & $0.38$ && $0.11$ & $0.33$ & $0.46$ && $0.23$ & $0.33$ & $0.38$ \\
    \GREEDYS{}  & $0.21$ & $0.30$ & $0.38$ && 
    $0.12$ & $0.32$ & $0.48$ && $0.27$ & $0.39$ & $0.43$\\
    Random & $0.12$ & $0.19$ & $0.23$  &&
    $0.10$ & $0.19$ & $0.28$ &&
    $0.12$ & $0.19$ & $0.23$\\
    \texttt{IIAB} & $0.07$ & $0.24$ & $0.34$ && 
    $0.11$ & $0.22$ & $0.41$ &&
    $0.07$ & $0.24$ & $0.34$ \\ \midrule
    Fail-Aware & $0.00$ & $0.02$ & $0.10$ && $0.00$ & $0.06$ & $0.18$ && $0.00$ & $0.02$ & $0.10$\\ \bottomrule
    \end{tabular}
\end{table}

As with UNOS graphs, results for \MCTSS{} and \GREEDYS{} are quite similar, and both methods achieve larger $\Delta^\texttt{MAX}$ than Random, \texttt{IIAB}, and Fail-Aware.
We make two observations: (1) \GREEDYS{} appears to achieve larger $\Delta^\texttt{MAX}$ than \MCTSS{} in the single-stage setting, likely because of insufficient training time for \MCTSS{}; (2) in the multi-stage setting, \MCTSS{} performs \emph{at least} as well as \GREEDYS{}, and often better.
Observation (2) is consistent with our experiments on UNOS graphs, and is somewhat surprising given that \MCTSS{} used less training time in these experiments.
This suggests that \MCTSS{} may substantially improve over \GREEDYS{} in the multi-stage setting; we leave further investigation to future work.

\section{Proofs for Section~\ref{sec:edge-selection-problem}}\label{app:proofs}

In the proofs of Proposition~\ref{prop:non-monotonic} and Proposition~\ref{prop:non-submodular} we consider a setting where all edges' pre-match rejections and post-match failures are i.i.d., where $P_R=0.5$ is the pre-match rejection probability, $P_Q=1.0$ is the post-match success probability if the edge is queried-and-accepted, and $P_N=0.5$ is the success probability if $e$ is not queried.
That is, queried edges have rejection probability $0.5$, accepted edges have zero failure probability, and non-queried edges have failure probability $0.5$.

\subsection{Proof of Proposition~\ref{prop:non-monotonic}}
(Proof by counterexample.)
We provide an example where querying a single edge results in a \emph{lower} objective value in Problem~\ref{eq:singlestage} (i.e., final expected \matching{} weight) than querying no edges--when using the max-weight \matching{} policy $M^{\texttt{MAX}}(\cdot)$.

\tikzstyle{altruist}=[circle,
  thick,
  minimum size=1.0cm,
  draw=black!50!green!80,
  fill=black!20!green!20
]

\tikzstyle{every path}=[line width=1pt]

\begin{figure}
\centering
\begin{tikzpicture}[>=latex]

  \tikzstyle{fake}=[rectangle,minimum size=8mm,opacity=0.0]
  
  \node (a) [shape=circle,
    draw=gray!80,
    line width=0.6mm,text=black,font=\bfseries,
    fill=gray!20,
    minimum size=1.2cm,
  ] at (0,0) {$A$};

  \node (b) [shape=circle,
    draw=gray!80,
    line width=0.6mm,text=black,font=\bfseries,
    fill=gray!20,
    minimum size=1.2cm,
  ] at (2,0) {$B$};
  
  \node (c) [shape=circle,
    draw=gray!80,
    line width=0.6mm,text=black,font=\bfseries,
    fill=gray!20,
    minimum size=1.2cm,
  ] at (4,0) {$C$};
  
    \node (d) [shape=circle,
    draw=gray!80,
    line width=0.6mm,text=black,font=\bfseries,
    fill=gray!20,
    minimum size=1.2cm,
  ] at (6,0) {$D$};
  
      \node (e) [shape=circle,
    draw=gray!80,
    line width=0.6mm,text=black,font=\bfseries,
    fill=gray!20,
    minimum size=1.2cm,
  ] at (3,-1.6) {$E$};
  
    \node (f) [shape=circle,
    draw=gray!80,
    line width=0.6mm,text=black,font=\bfseries,
    fill=gray!20,
    minimum size=1.2cm,
  ] at (5,-1.6) {$F$};
 
  \draw[->] (a) edge [bend left=20] node[midway,above=0.3cm,rotate=0] {$e_1$}(b);
  \draw[->] (b) edge [bend left=20] (a);
  
  \draw[->] (b) edge [bend left=20] node[midway,above=0.3cm,rotate=0] {$e_2$} (c);
  \draw[->] (c) edge [bend left=20] (e);
  \draw[->] (e) edge [bend left=20] node[midway,below left,rotate=0] {$\bm w_{(E, B)}=1.5$} (b);

  \draw[->] (c) edge [bend left=20] node[midway,above=0.3cm,rotate=0] {$e_3$} (d) ;
  \draw[->] (d) edge [bend left=20] (f);
  \draw[->] (f) edge [bend left=20] (c);
  

\end{tikzpicture} 
\caption{
Exchange graph for Propositions~\ref{prop:non-monotonic} and~\ref{prop:non-submodular}.
All edges have weight $1$ except for edge $(E, B)$, which has weight 1.5.
}\label{fig:proof-graph}
\end{figure}
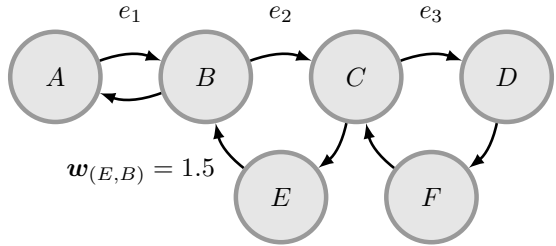

Consider the exchange graph in Figure~\ref{fig:proof-graph}; edge $(E, B)$ has weight 1.5, while all other edges have weight $1$.
First we consider the objective due to querying no edges, $V^S(\bm 0)$. 
In this case, no edges can be rejected pre-match, the max-weight \matching{} includes cycle $(C, D, F)$ (expected weight $3\times (1/2)^3 = 3/8$) and cycle $(A, B)$ (expected weight $2 \times (1/2)^2 = 1/2$), with total expected \matching{} weight $7/8$. 
That is, $V^S(\bm 0)=7/8$.

Next consider the objective due to querying only edge $e_3=(C, D)$, and let  $\bm q'$ denote edge set $\{e_3\}$.
With probability $1/2$, $e_3$ is rejected and cycle $(B, C, E)$ is the max-weight \matching{} -- with expected weight $3.5/8$. 
With probability $1/2$, $e_3$ is accepted and the max-weight \matching{} includes cycles $(A, B)$ (with expected weight $1/2$) and $(C, D, F)$ (with expected weight $3/4$); this \matching{} has total expected weight $5/4$.
Thus, $V^S(\bm q) =27/32 < 7/8 = V^S(\bm 0)$, which concludes the proof. 

\subsection{Proof of Proposition~\ref{prop:non-submodular}}
(Proof by counterexample.) 
We provide an example where the objective value in Problem~\ref{eq:singlestage} (i.e., final expected \matching{} weight) is non-submodular--when using the max-weight \matching{} policy $M^{\texttt{MAX}}(\cdot)$.
We use the same rejection and failure distribution as in the proof of Proposition~\ref{prop:non-monotonic}.

Consider the exchange graph in Figure~\ref{fig:proof-graph}; edge $(E, B)$ has weight 1.5, while all other edges have weight $1$.
With some abuse of notation, we will denote by $V^S(\{e_a, \dots, e_N\})$ the objective of Problem~\ref{eq:singlestage} due to edge set $\{e_a, \dots, e_N\}$.
Our counterexample for submodulartiy is that, for this graph, 
$$V^S(X \cup \{e_1, e_2\}) + V^S(X) > V^S(X \cup \{ e_1\}) + V^S(X \cup \{e_2\}),$$
with set $X\equiv \{e_3\}$.
That is, the objective increase due to of querying \emph{both} edges $e_1$ and $e_3$ is greater than the combined increase due to querying both edges separately.
Next we explicitly calculate each of the above terms.

\paragraph{$V^S(X)=V^S(\{e_3\})$.}
There are two cases to consider:
\begin{itemize}
    \item $e_3$ is accepted, with probability $1/2$. The max-weight \matching{} is cycles $(A, B)$ and $(C, D, F)$, with expected weight $(1/2 + 3/4)$,
    \item $e_3$ is rejected, with probability $1/2$. The max-weight \matching{} is cycle $(B, C, E)$, with expected weight $3.5/8$.
\end{itemize}
Thus, $V^S(X)=(1/2)(1/2 + 3/4) + (1/2)(3.5/8)= 27/32$.


\paragraph{$V^S(X\cup \{e_1\})=V^S(\{e_1, e_3\})$.} There are four cases to consider:
\begin{itemize}
    \item $e_1$ and $e_3$ are accepted, with probability $1/4$. The max-weight \matching{} is cycles $(A, B)$ and $(C, D, F)$, with expected weight $(1 + 3/8)$,
    \item $e_1$ is rejected and $e_3$ is accepted, with probability $1/4$. The max-weight \matching{} is cycle $(B, C, E)$, with expected weight $3.5/8$.
    \item $e_1$ is accepted and $e_3$ is rejected, with probability $1/4$. The max-weight \matching{} is cycle $(B, C, E)$, with expected weight $3.5/8$.
    \item $e_1$ and $e_3$ are rejected, with probability $1/4$. The max-weight \matching{} is cycle $(B, C, E)$, with expected weight $3.5/8$.
\end{itemize}
Thus the objective is $V^S(X\cup \{e_3\})=(1/4)(1 + 3/8) + (3/4)(3.5/8) =43/64$.

\paragraph{$V^S(X\cup \{e_2\})=V^S(\{e_2, e_3\})$.} There are three cases to consider
\begin{itemize}
    \item $e_3$ is accepted, with probability $1/2$. The max-weight \matching{} is cycles $(A, B)$ and $(C, D, F)$, with expected weight $(1/2 + 3/4)$,
    \item $e_3$ is rejected and $e_3$ is accepted, with probability $1/4$. The max-weight \matching{} is cycle $(B, C, E)$, with expected weight $3.5/4$,
    \item $e_3$ and $e_2$ are rejected, with probability $1/4$. The max-weight \matching{} is cycle $(A,B)$, with expected weight $1/2$.
\end{itemize}
Thus the objective is $V^S(X\cup \{e_2\})=(1/2)(1/2 + 3/4) + (1/4)(3.5/4) + (1/4)(1/2)= 31/32$.

\paragraph{$V^S(X\cup \{e_1, e_2\})=V^S(\{e_1, e_2, e_3\})$.} There are four cases to consider:
\begin{itemize}
    \item $e_1$ and $e_3$ are accepted, with probability $1/4$. The max-weight \matching{} is cycles $(A, B)$ and $(C, D, F)$, with expected weight $(1 + 3/4)$,
    \item $e_1$ is accepted and $e_2$ is rejected, with probability $1/4$ (the response from $e_3$ is irrelevant). The max-weight \matching{} is $(A, B)$ and $(C, D, F)$, with expected weight $1 + 3/8$.
    \item $e_1$ is rejected and $e_2$ is accepted (the response from $e_3$ is irrelevant), with probability $1/4$. The max-weight \matching{} is cycle $(B, C, E)$, with expected weight $3.5/4$.
    \item $e_1$ and $e_2$ are rejected (the response from $e_3$ is irrelevant), with probability $1/4$. The max-weight \matching{} is cycle $(C, D, F)$, with expected weight $3/8$.
\end{itemize}
Thus the objective is $V^S(X\cup \{e_1, e_2\}) = (1/4) (1 + 3/4) + (1/4)(1 + 3/8) + (1/4) (3.5/4) + (1/4)(3/8)=35/32$.

Finally, we have:
\begin{align*}
    V^S(X\cup \{e_1, e_2\}) + V^S(X) &= 35/32 + 27/32\\
    &= 1.9375
\end{align*}
and
\begin{align*}
    V^S(X\cup \{e_1\}) + V^S(X\cup \{e_2\}) &= 43/64 + 31/32 \\
    &= 1.640625
\end{align*}
Therefore, $V^S(X\cup \{e_1, e_2\}) + V^S(X)> V^S(X\cup \{e_1\}) + V^S(X\cup \{e_2\})$, which concludes the proof.

\subsection{Proof of Proposition~\ref{prop:fail-aware-monotonic}}

For the proof of Proposition~\ref{prop:fail-aware-monotonic} we make one assumption about the distribution of edge rejections and failures: querying \emph{additional} edges cannot increase the overall probability of rejection or failure for any edge. 
%
%
\begin{assumption}\label{assumption:additional-queries}
Let $\bm q, \bm r \in \{0, 1\}^{|E|}$ denote initial edge queries and responses.
Let $\bm q'$ be \emph{additional} edges, such that $\bm q + \bm q'\in \{0, 1\}^{|E|}$ denotes an augmented edge set; let $\bm r'\in \{0, 1\}^{|E|}$ denote the responses to edges $\bm q'$ only.
We assume that for any such $\bm q$, $\bm r$, and $\bm q'$,
\begin{equation*}
    \expect\left[\bm r + \bm f \mid \bm q, \bm r \right] \geq 
    \expect\left[\bm r + \bm r' + \bm f \mid \bm q + \bm q', \bm r \right] \; .
\end{equation*}
\end{assumption}
Intuitively, Assumption~\ref{assumption:additional-queries} excludes distributions where queries arbitrarily increase edge failure or rejection. 
For example, Assumption~\ref{assumption:additional-queries} disallows the following distribution: suppose all edges are independent; all queried edges are accepted ($P(\bm r_e=1 \mid \bm q)=0$ for all $\bm q$), all accepted edges have failure probability $0.5$ ($P(\bm f_e=1 \mid \bm q_e=1, \bm r_e=0)=0.5$), and all non-queried edges have failure probability $0.1$ ($P(\bm f_e=1 \mid \bm q_e= \bm r_e=0)=0.1$).
In this case, if an edge is not queried, then it has overall rejection or failure probability $0.1$ (i.e., $\expect [ \bm r_e + \bm f_e \mid \bm q, \bm r]=0.1$ with $\bm q_e=0$); if this edge is queried, then it has rejection or failure probability $0.5$ (i.e., $\expect [ \bm r_e + \bm r_e' + \bm f_e \mid \bm q + \bm q', \bm r]=0.5$ with $\bm q_e'=1$).

First we prove a handful of useful results.

\begin{definition}[Edge Independence]
Two edges $e, e'\in E$ are \emph{independent} if (a) their rejection distributions are conditionally independent, given whether or not they were queried:
$$\bm r_e \indep \bm r_{e'} \mid \bm q_e \quad \text{and} \quad \bm r_e \indep \bm r_{e'} \mid \bm q_{e'}$$
and (b) their failure distributions are conditionally independent, given whether or not they were queried and rejected:
$$\bm f_e \indep \bm f_{e'} \mid \bm q_e, \bm r_e \quad \text{and}\quad  \bm f_e \indep \bm f_{e'} \mid \bm q_{e'}, \bm r_{e'}\, .$$
\end{definition}

\begin{lemma}\label{lem:cycle-weights}
If all edges are independent, then additional edge queries cannot decrease expected post-match cycle and chain weights.
Formally,
\begin{equation*}
    \expect\left[F(c,\bm r + \bm f) \mid \bm q, \bm r \right] \leq 
    \expect\left[F(c,\bm r + \bm r' + \bm f) \mid \bm q + \bm q', \bm r \right] 
\end{equation*}
for any $\bm q , \bm q'\in \{0, 1\}^{|E|}$ such that $\bm q + \bm q'\in \{0, 1\}^{|E|}$, for any $\bm r \in \{0, 1\}^{|E|}$, and for all $c \in \mathcal C$.
\end{lemma}
\begin{proof}
We address cycles and chains separately.

\paragraph{Cycles.} Conditional on fixed $\bm q$ and $\bm r$, the expected weight of cycle $c=(e_1, \dots, e_L)$ is expressed as
\begin{align*}
    \expect \left[ F(c,\bm r + \bm f) \mid \bm q, \bm r\right] &= \left( \sum_{e\in c} w_e \right) \expect \left[ \prod_{e\in c} (1 - \bm r_e - \bm f_e) \mid \bm q, \bm r \right] \\
    &= \left( \sum_{e\in c} w_e \right) \prod_{e\in c}\left( 1 -  \expect \left[\bm r_e + \bm f_e \mid \bm q, \bm r \right] \right)
\end{align*}
where the second step is due to the fact that all $\bm f_e$ are independent.
Similarly, for fixed $\bm q'$,
\begin{align*}
     \expect\left[F(c,\bm r + \bm r' + \bm f) \mid \bm q + \bm q', \bm r \right] &= \left( \sum_{e\in c} w_e \right) \prod_{e\in c}\left( 1 -  \expect \left[\bm r_e + \bm r_e' + \bm f_e \mid \bm q + \bm q', \bm r \right] \right) \; .
\end{align*}
Due to Assumption~\ref{assumption:additional-queries}, the following inequality holds for all edges $e \in E$
$$
\expect \left[\bm r_e + \bm f_e \mid \bm q, \bm r \right] \geq \expect \left[\bm r_e + \bm r_e' + \bm f_e \mid \bm q + \bm q', \bm r \right] \;,
$$
and it follows that
$$
\expect\left[F(c,\bm r + \bm f) \mid \bm q, \bm r \right] \leq \expect\left[F(c,\bm r + \bm r' + \bm f) \mid \bm q + \bm q', \bm r \right].
$$

\paragraph{Chains.} Similarly, the expected weight of chain $c=(e_1, \dots, e_L)$ is expressed as
\begin{align*}
    \expect \left[ F(c,\bm r + \bm f) \mid \bm q, \bm r\right] &= \sum_{k=1}^L \left(\sum_{j=1}^k w_j \right) \expect \left[ \prod_{j=1}^k (1 - \bm r_{e_j} - \bm f_{e_j}) \mid \bm q, \bm r \right] \\
    &= \sum_{k=1}^L \left(\sum_{j=1}^k w_j \right) \prod_{j=1}^k \left(1 - \expect \left[ \bm r_{e_j} + \bm f_{e_j} \mid \bm q, \bm r \right] \right) \;,
\end{align*}
where the second step is due to the fact that $\bm f_e$ are independent.
Similarly,
\begin{align*}
    \expect \left[ F(c,\bm r + \bm r' + \bm f) \mid \bm q + \bm q', \bm r\right]
    &= \sum_{k=1}^L \left(\sum_{j=1}^k w_j \right) \prod_{j=1}^k \left(1 - \expect \left[ \bm r_{e_j} + \bm r_{e_j}' + \bm f_{e_j} \mid \bm q + \bm q', \bm r \right] \right) \;.
\end{align*}
as before, due to Assumption~\ref{assumption:additional-queries} it follows that
$$
\expect\left[F(c,\bm r + \bm f) \mid \bm q, \bm r \right] \leq \expect\left[F(c,\bm r + \bm r' + \bm f) \mid \bm q + \bm q', \bm r \right].
$$

\end{proof}

\begin{lemma}\label{lem:add-edge-improvement}
With a failure-aware \matching{} policy, and if all edges are independent, adding a single edge to any edge query set weakly improves the objective of Problem~\ref{eq:singlestage}.
Formally, for any $\bm q, \bm q'\in \{0, 1\}^{|E|}$ with $\bm q + \bm  q' \in \{0, 1\}^{|E|}$ and $|\bm q'|=1$, and $M(\bm r) \equiv M^\texttt{FA}(\bm r)$,
$$ V^S(\bm q) \leq V^S(\bm q + \bm q')$$ 
\end{lemma}
\begin{proof}
The objective of Problem~\ref{eq:singlestage} for edge set $\bm q$ is expressed as
\begin{align*}
V^S(\bm q) &= \expect_{\bm r \mid \bm q} \left[ \expect_{\bm f \mid \bm q, \bm r}\left[\sum_{c\in \mathcal C} M^\texttt{FA}_c(\bm r) F(c, \bm r + \bm f) \right] \right] \\
&= \sum_{\bm r \in \{0, 1\}^{|\bm q|}} P_{\bm q}(\bm r) \expect_{\bm f \mid \bm q, \bm r}\left[\sum_{c\in \mathcal C} M^\texttt{FA}_c(\bm r) F(c, \bm r + \bm f) \right]\\
&= \sum_{\bm r \in \{0, 1\}^{|\bm q|}}  P_{\bm q}(\bm r) \,\sum_{c\in \mathcal C} M^\texttt{FA}_c(\bm r)  \expect_{\bm f \mid \bm q, \bm r}\left[ F(c, \bm r + \bm f) \right]
\end{align*}
For edge set $\bm q + \bm q'$ we partition response variables into $\bm r, \bm r'\in \{0, 1\}^{|E|}$, where $\bm r_e$ is the response variable for all edges $e\in \bm q$, and $\bm r_e=0$ for all other edges (including the edge in $\bm q'$).
Similarly, $\bm r_e'$ is the response variable for edge $\bm q'$, and $\bm r_e'=0$ for all other edges.
The objective of $\bm q + \bm q'$ is expressed as
\begin{align*}
V^S(\bm q + \bm q') &= \expect_{\bm r, \bm r' \mid \bm q + \bm q'} \left[ \expect_{\bm f \mid \bm q + \bm q', \bm r + \bm r' }\left[\sum_{c\in \mathcal C} M^\texttt{FA}_c(\bm r + \bm r') F(c,\bm r + \bm r' + \bm f) \right] \right] \\
&= \sum_{\bm r \in \{0, 1\}^{|\bm q|}} \; P_{\bm q + \bm q'}(\bm r) \expect_{\bm r' \mid \bm q + \bm q'} \left[ \expect_{\bm f \mid \bm q + \bm q', \bm r + \bm r' }\left[\sum_{c\in \mathcal C} M^\texttt{FA}_c(\bm r + \bm r')^\top F(c, \bm r + \bm r' + \bm f) \right] \right] \\
&= \sum_{\bm r \in \{0, 1\}^{|\bm q|}} \; P_{\bm q}(\bm r) \expect_{\bm r' \mid \bm q + \bm q'} \left[ \expect_{\bm f \mid \bm q + \bm q', \bm r + \bm r' }\left[ \sum_{c\in \mathcal C} M^\texttt{FA}_c(\bm r + \bm r') F(c,\bm r + \bm r' + \bm f) \right] \right],
\end{align*}
where in the final line we replace $P_{\bm q + \bm q'}(\bm r) $ with $P_{\bm q}(\bm r)$, because each $\bm r_e$ is conditionally independent, given $\bm q_e$.

Next, by definition
$$\expect_{\bm f \mid \bm q + \bm q', \bm r + \bm r'}\left[\sum_{c\in \mathcal C} M^\texttt{FA}_c(\bm r+ \bm r') F(c, \bm r + \bm r' + \bm f) \right]  \geq \expect_{\bm f \mid \bm q + \bm q', \bm r + \bm r'}\left[\sum_{c\in \mathcal C} \bm x_c F(c, \bm r + \bm r' + \bm f) \right] \quad \forall \bm x \in \mathcal M.$$
That is, $M^{FA}$ is guaranteed to maximize this expectation, and thus
\begin{align*}
V^S(\bm q + \bm q') &\geq \sum_{\bm r \in \{0, 1\}^{|\bm q|}} \; P_{\bm q}(\bm r) \expect_{\bm r' \mid \bm q + \bm q'} \left[ \expect_{\bm f \mid \bm q + \bm q', \bm r + \bm r' }\left[\sum_{c\in \mathcal C} M^\texttt{FA}_c(\bm r) F(c, \bm r + \bm r' + \bm f) \right] \right] \tag{B} \\
&= \sum_{\bm r \in \{0, 1\}^{|\bm q|}} \; P_{\bm q}(\bm r) \, \sum_{c\in \mathcal C} M^\texttt{FA}_v(\bm r)  \expect_{\bm r' \mid \bm q + \bm q'} \left[ \expect_{\bm f \mid \bm q + \bm q', \bm r + \bm r' }\left[ F(c, \bm r + \bm r' + \bm f) \right] \right] \tag{C}
\end{align*}
Finally, combining (B) and (C) with Lemma~\ref{lem:cycle-weights}, the following inequality holds
$$ V^S(\bm q) \leq V^S(\bm q + \bm q'). $$

\end{proof}

Using the above lemmas, the proof of Proposition~\ref{prop:fail-aware-monotonic} is straightforward:
\paragraph{Proposition~\ref{prop:fail-aware-monotonic}} 
\emph{
With a failure-aware \matching{} policy, if all edges are independent, the objective of Problem~\ref{eq:singlestage} is monotonic in the set of queried edges.
}
\begin{proof}
Let $\bm q' , \bm q''\in \mathcal E$ be two edge sets such that $\bm q' \subseteq \bm q''$.
It remains to show that, with \matching{} policy $M(\bm r)\equiv M^\texttt{FA}(\bm r)$, 
$$ V^S(\bm q'') \leq V^S(\bm q').$$

First note that because $\mathcal E$ is a matroid, there is a sequence of edges $(\bm q^{e_1}, \dots, \bm q^{e_L})$ (with each $|\bm q^{e_i}|=1$) such that $\bm q'' + \bm q^{e_1} + \dots + \bm q^{e_L}=\bm q'$.
Due to Lemma~\ref{lem:add-edge-improvement}, the following sequence of inequalities hold:
\begin{align*}
    V(\bm q'') \leq V(\bm q'' + \bm q^{e_1}) \\
    &  \leq V(\bm q'' + \bm q^{e_1} + \bm q^{e_2}) \\
    & \dots \\
    &  \leq V(\bm q'' + \bm q^{e_1} + \dots + \bm q^{e_L}) \\
    &= V(\bm q')
\end{align*}
which concludes the proof.
\end{proof}

\section{Algorithm Descriptions}\label{app:alg}

Here we describe more explicitly the algorithms for \GREEDYS{} and \MCTSS{}, for both the single-stage and multi-stage settings.

\subsection{UCB Value Estimates for \MCTSS{}}

Both the single- and multi-stage versions of \MCTSS{} use the method of~\cite{kocsis2006bandit} to select the next child node to explore.
The formula used to estimate a node's UCB value is
$$ \frac{\frac{U}{N} - V^{min}}{V^{max}- V^{min}} + \sqrt{N^P/N} $$
where $U$ is the ``UCB value estimate'' calculated by \MCTSS{}, $N$ is the number of visits to the node, $N^P$ is the number of visits to the node's parent, and $V^{max}$ and $V^{min}$ are the largest and smallest \emph{node values} encountered during search.
In single-stage \MCTSS{}, all nodes have both a \emph{node value} (the objective value of Problem~\ref{eq:singlestage}) and a UCB value estimate; as described below, in multi-stage \MCTSS{} only query nodes have a UCB value estimate, and only leaf nodes have a \emph{node value} (expected matched weight, after observing responses from all queried edges).

\subsection{Greedy Single-Stage Edge Selection}\label{sec:alg}
Algorithm~\ref{alg:greedy-ss} gives a pseudocode description of \GREEDYS{} for the single-stage setting.

\begin{algorithm}
\SetAlgoNoLine
\DontPrintSemicolon
(input) $\mathcal E$: legal edge sets\;
\;
$\bm q^R \gets \bm 0$ \quad the root node (no edges)\;
$V^* \gets $ objective value of $\bm q^R$ Problem~\ref{eq:singlestage} \;
\While{$\bm q^R$ has children}{
    $\bm q'\gets$ child node of $\bm q^R$ with maximal objective value in Problem~\ref{eq:singlestage} \;
    $\bm q^R \gets \bm q'$\;
  }
\KwRet $\bm q^R$
\caption{\GREEDYS{}: Greedy Search Heuristic for Single-Stage Edge Selection}
\label{alg:greedy-ss}
\end{algorithm}

\subsection{Multi-Stage Edge Selection}\label{app:multistage-alg}

In the following sections we describe multi-stage versions of \MCTSS{} and \GREEDYS{}.
Unlike in the single-stage setting, these algorithms take as input a set of previously-queried edges $\bm q\in \{0, 1\}^{|E|}$ and a corresponding set of observed rejections $\bm r \in \{0, 1\}^{|E|}$; they output the \emph{next} edge to query.

\paragraph{Multi-Stage \MCTSS{}.}

The multi-stage search tree is somewhat more complicated than in the single-stage setting, as each node in the search tree corresponds to both a set of queried edges and a set of observed rejections. 
For this purpose we use two types of nodes: \emph{outcome} nodes, and \emph{query} nodes.
Outcome nodes consist of previously-queried edges $\bm q$ and previously-observed rejections $\bm r$, and are represented by tuple $(\bm q, \bm r)$.
(The root of the search tree corresponds to \emph{no} queries or observed rejections, $(\bm 0, \bm 0)$.)
The children of an outcome node are \emph{query} nodes, represented by the next edge to query from the parent (outcome), represented by tuple $(\bm q, \bm r, e)$.
Each outcome node has one child for every edge that has not yet been queried:
$$C^O(\bm q, \bm r)\equiv \left\{ (\bm q, \bm r, e) \; \mid \; \forall e\in E: \; \bm q + \bm u^e \in \mathcal E \right\} $$
where $\bm u^e$ is the unit vector for element $e$ ($\bm u^e_i=0$ for all $i\neq e$, and $\bm u^e_e=1$).
Each query node has exactly two children: one where the queried edge is accepted, and one where the queried edge is rejected,
$$C^Q(\bm q, \bm r, e) \equiv \left\{ (\bm q + \bm u^e, \bm r),  (\bm q + \bm u^e, \bm r + \bm u^e) \right\} $$
As before, the \emph{level} of a node refers to the number of queried edges: $|\bm q|$ for outcome nodes, and $|\bm q|+1$ for query nodes.

As before we refer to nodes with no children as leaf nodes; note that only outcome nodes are leaf nodes.
Unlike the single-stage version of \MCTSS{}, in the multi-stage setting we only consider the value of leaf nodes\footnote{This decision was made in part because initial results indicate that edge selection is essentially monotonic.}.
The value of a leaf (outcome) is
$$ V^O(\bm q, \bm r) \equiv W(M(\bm r); \bm q, \bm r),$$
where as before $M(\bm r)$ denotes the matching policy, and $W(\bm x; \bm q, \bm r)$ denotes the expected matching weight of $\bm x$, subject to $\bm q$ and $\bm r$.
The value of leaf outcome nodes is used to by \texttt{QSample} and \texttt{OSample} to guide multi-stage \MCTSS{}. 

%

Algorithm~\ref{alg:uct-multi-stage} describes the multi-stage version of \MCTSS{}, taking previously-queried edges and observed responses as input.
This algorithm initializes the value estimate $U[\cdot]$ and number of visits $N[\cdot]$ for query nodes in the next $L$ levels--these quantities are used in the UCB calculation.

\begin{algorithm}[H]
\SetAlgoNoLine
\DontPrintSemicolon
(input) $\mathcal E$: legal edge sets\;
(input) $K$: maximum size of any legal edge set\;
(input) $T$: time limit\;
(input) $L$: number of look-ahead levels\;
(input) $\bm q^R$: previously-queried edges \;
(input) $\bm r^R$: previously-observed rejections \;
\;

$M \gets \min\{N + L, K \}$ \; 
$Q \gets$ all query nodes which are descendants of $(\bm q^R, \bm r^R)$, up to level $M$ \;
$U[(\bm q, \bm r, e)] \gets 0 \; \forall (\bm q, \bm r, e)\in Q$  \quad UCB value estimate \;
$N[(\bm q, \bm r, e)] \gets 0 \;\forall \bm q\in Q$ \quad number of visits \;
\While{less than time $T$ has passed}{
\texttt{QSample}($\bm q^R$, $\bm r^R$ $M$)\;
}
$(\bm q^R, \bm r^R, e^*) \gets$ child node of $(\bm q^R, \bm r^R)$ with the greatest UCB estimate \;
\KwRet $e^*$
\caption{Multi-Stage \MCTSS{}}
\label{alg:uct-multi-stage}
\end{algorithm}

\begin{algorithm}[H]
\SetAlgoNoLine
\DontPrintSemicolon
(input) $(\bm q, \bm r)$: outcome node \;
(input) $M$: maximum level to sample from
\;
\;
\If{$(\bm q, \bm r)$ has no children}{
\KwRet $V^O(\bm q, \bm r)$ \quad (return the value of this outcome node)}
\If{$(\bm q, \bm r)$ has children}{
\eIf{$|\bm q|<M - 1$}{
    $(\bm q, \bm r, e') \gets$  child node of $(\bm q, \bm r)$ with the greatest UCB estimate\;
   \texttt{OSample}($\bm q$, $\bm r$, $e$)
}{
$(\bm q', \bm r') \gets$ random leaf node, descendant from $(\bm q, \bm r)$ \;
\KwRet{$V^O(\bm q', \bm r')$} \;
}
}
\caption{\texttt{QSample}: Function for sampling query nodes in multi-stage \MCTSS{}}\label{alg:sample-multistage}
\end{algorithm}

\begin{algorithm}[H]
\SetAlgoNoLine
\DontPrintSemicolon
(input) $(\bm q, \bm r, e)$: query node \;
\;
$N[(\bm q, \bm r, e)] \gets N[(\bm q, \bm r, e)] + 1$ \;
$\bm q' \gets\bm q + \bm u^e$ (new query vector with edge $e$ added) \;
$Z \gets$ randomly sample a response to edge $e$ (0 if accept, 1 if reject)\;
$\bm r' \gets \bm r + Z \bm u^e$ \quad (updated rejection vector)\;
$U[(\bm q, \bm r, e)] \gets U[(\bm q, \bm r, e)] + \texttt{QSample}(\bm q', \bm r')$ \;
\caption{\texttt{OSample}: Function for sampling outcome nodes in multi-stage \MCTSS{}}\label{alg:o-sample-multistage}
\end{algorithm}

Algorithm~\ref{alg:sample-multistage} (\texttt{QSample})  samples query nodes from an outcome node, while Algorithm~\ref{alg:o-sample-multistage} (\texttt{OSample}) samples outcome nodes from a query node (and updates the query node's UCB value estimate).

\paragraph{Multi-Stage \GREEDYS{}.}
Algorithm~\ref{alg:multi-stage-greedy} gives a pseudocode description of the multi-stage version of \GREEDYS{}.
This search heuristic returns the next edge to query with the highest expected final \matching{} weight, \emph{ignoring all future queries}.  
In other words, this approach treats every edge as the \emph{last} edge; one might call this heuristic ``myopic'' as well as greedy.
\begin{algorithm}
\SetAlgoNoLine
\DontPrintSemicolon
(input) $\mathcal E$: legal edge sets\;
(input) $\bm q$: previously-queried edges\;
(input) $\bm r$: previously-observed rejections\;
\;
$e^* \gets \emptyset$
$V^* \gets 0$ \;
\For{all $\bm q'$ in $\bm q$'s children}{
    $e' \gets$ the new edge queried in child node $\bm q'$\;
    $\bm r^A \gets \bm r$\;
    $\bm r^R \gets \bm r$\;
    $\bm r^A_{e'} \gets 0$ \quad(response scenario where $e'$ is accepted, and $\bm r_{e'}=0$) \;
    $\bm r^R_{e'} \gets 1$ \quad (response scenario where $e'$ is rejected, and $\bm r_{e'}=1$) \;
    $p^A \gets$ probability that $e$ is accepted, conditional on previous responses \;
    $p^R \gets$ probability that $e$ is rejected, conditional on previous responses \;
    $V' \gets p^A \cdot W(M(\bm r^A); \bm q', \bm r^A) + p^R W(M(\bm r^R); \bm q', \bm q^R)$ \quad (value of querying edge $e'$)\;
    \If{$V' > V^*$}{
    $\bm e^* \gets e'$ \;
    $V^* \gets V'$ \;
    }
  }
\KwRet $\bm e^*$
\caption{\GREEDYS{} Heuristic for Multi-Stage Edge Selection}
\label{alg:multi-stage-greedy}
\end{algorithm}

\end{document}